\documentclass{article}

\PassOptionsToPackage{sort}{natbib}

\usepackage[accepted]{icml2025}

\newif\ifarXiv
\arXivfalse

\usepackage{hyperref}
\usepackage{url}

\usepackage[OT1]{fontenc}

\usepackage{amsthm}
\usepackage{amsmath}
\usepackage{amsfonts}
\usepackage[american]{babel}
\usepackage{graphicx}
\usepackage{microtype}
\usepackage{tikz}
\usepackage{pgfplots}
\usepackage{caption}
\usepackage{subcaption}
\usepackage{wrapfig}
\usepackage{array}
\usepackage{booktabs}
\usepackage{paralist}
\usepackage{nicefrac}
\usepackage{xcolor}
\usepackage[nameinlink, capitalise]{cleveref}
\usepackage{thmtools}
\usepackage{thm-restate}

\usepackage{siunitx}
\definecolor{CBGreen}{RGB}{0, 158, 115}
\definecolor{cardinal} {RGB}{196, 30, 58}
\definecolor{bleu}     {RGB}{ 49,140,231}

\hypersetup{%
  colorlinks = true,
  urlcolor   = CBGreen,
  citecolor  = CBGreen,
  linkcolor  = CBGreen,
}

\pgfplotsset{
  compat = 1.17
}

\usepgfplotslibrary{colorbrewer}
\usepgfplotslibrary{groupplots}
\usepgfplotslibrary{statistics}

\usetikzlibrary{external}
\usepackage{bbm}

\newtheorem{definition}{Definition}

\newtheorem{corollary}{Corollary}

\DeclareMathOperator{\ECT}{ECT}
\DeclareMathOperator{\lECT}{\ell -ECT}

\newcommand{\norm}[1]{\left\lVert#1\right\rVert}

\begin{document}

\twocolumn[
\icmltitle{Diss-l-ECT: Dissecting Graph Data with Local Euler Characteristic Transforms}



\icmlsetsymbol{equal}{*}

\begin{icmlauthorlist}

\icmlauthor{Julius von Rohrscheidt}{1,2}
\icmlauthor{Bastian Rieck}{1,2,3}

\end{icmlauthorlist}

\icmlaffiliation{1}{Institute of AI for Health, Helmholtz Munich, Germany}
\icmlaffiliation{2}{Technical University of Munich, Germany}
\icmlaffiliation{3}{University of Fribourg, Switzerland}

\icmlcorrespondingauthor{Julius von Rohrscheidt}{julius.rohrscheidt@helmholtz-munich.de}
\icmlcorrespondingauthor{Bastian Rieck}{bastian.grossenbacher@unifr.ch}

\icmlkeywords{Machine Learning, ICML}

\vskip 0.3in
]

\printAffiliationsAndNotice{}

\begin{abstract}
  \emergencystretch=3em
  \hyphenpenalty   =10000
  The Euler Characteristic Transform~(ECT) is an
  efficiently-computable geometrical-topological invariant that
  characterizes the \emph{global} shape of data. 
  In this paper, we introduce the \emph{Local Euler Characteristic
  Transform}~($\lECT$), a novel extension of the ECT particularly
  designed to enhance expressivity and interpretability in graph
  representation learning.
  Unlike traditional graph neural networks~(GNNs), which may lose
  critical local details through aggregation, the $\lECT$ provides
  a lossless representation of local neighborhoods.
  This approach addresses key limitations in GNNs by preserving local
  structures while maintaining global interpretability.
  Moreover, we construct a rotation-invariant metric based on $\lECT$s
  for spatial alignment of data spaces.
  Our method exhibits superior performance compared to standard GNNs
  on a variety of benchmark node-classification tasks, while also
  offering theoretical guarantees that demonstrate its effectiveness.
\end{abstract}

\section{Introduction}

\emergencystretch=3em
\hyphenpenalty=1000

Traditionally, graph neural network~(GNNs) rely on message-passing
schemes to aggregate node features. While effective for many tasks, this
approach often leads to the loss of critical \emph{local} information, as the
aggregation process can diffuse and obscure the original node vector
representations~\citep{topping2022oversquashing}.
This limitation makes it challenging to preserve local characteristics
that may be essential for some applications.
To address this limitation, we harness the \emph{Euler Characteristic
Transform}~\citep[$\ECT$]{Turner14a}, an expressive geometrical-topological
invariant that can be computed efficiently.
Relying only on~(weighted) sums, the $\ECT$ can be computed efficiently,
making it a powerful tool for representation
learning~\citep{roell2023differentiable}.
Moreover, the $\ECT$ is known to be \emph{invertible} for data
in~$\mathbb{R}^n$, ensuring that the original data can always be
reconstructed~\citep{curry2022many, ghrist2018persistent}.
In this paper, we extend the ECT to \emph{local} neighborhoods,
presenting the \emph{Local Euler Characteristic Transform}~($\lECT$),
a novel method designed to preserve \emph{local} structure while
retaining \emph{global} interpretability. The $\lECT$ captures both
topological~(i.e., structural) and geometrical~(i.e., spatial)
information around each data point, making it particularly
advantageous for graph-based or higher-order data.\footnote{%
  Our experiments deliberately focus on graphs, but we note that the
  method can be extended to novel higher-order datasets based on
  \emph{simplicial complexes}, for instance~\citep{Ballester25a}.
}
The $\lECT$ thus becomes an expressive \emph{fingerprint} of local
neighborhoods, specifically addressing the challenge of neighborhood
aggregation in featured graphs while ensuring the \emph{lossless
representation} of local node neighborhoods.
We theoretically investigate how $\lECT$s maintain critical local
details, and therefore provide a nuanced representation that can be
used for downstream graph-learning tasks such as node classification.
Our method is highly effective, particularly for tasks where node
feature aggregation may obscure essential differences, such as in graphs
with high heterophily. Additionally, the $\lECT$ framework's natural
vector representation makes it compatible with a wide range of
machine-learning models, facilitating both performance and
interpretability.

\textbf{As our main contributions,} we
\begin{inparaenum}[(i)]
    \item construct $\lECT$s in the context of embedded simplicial
      complexes~(and graphs) and theoretically investigate their
      expressivity in the special case of featured graphs,
    \item empirically show that this expressivity positions $\lECT$s as
      a powerful general tool for interpretable node classification,
      often superior to standard GNNs, and
    \item introduce an efficiently computable rotation-invariant metric
      based on $\lECT$s that facilitates the spatial alignment of
      geometric graphs.
\end{inparaenum}

\section{Background}

We define our method in the most general setting, i.e., that of
a simplicial complex, while also providing a brief introduction to graph
neural networks.

\paragraph{Simplicial Complexes}  
A \emph{simplicial complex} $K$ is a mathematical structure that
generalizes graphs to model higher-order~(non-dyadic) relationships and interactions.
While graphs model pairwise~(dyadic) connections between entities using nodes
and edges, simplicial complexes extend this representation to higher
dimensions by including triangles~(2-simplices),
tetrahedra~(3-simplices), and their higher-dimensional analogues.
Let $v_0,\dots,v_k\in\mathbb{R}^n$ be \emph{affinely independent} points.  
The \emph{(geometric) $k$-simplex} determined by these vertices is the convex hull  
\begin{equation*}
  \sigma = [v_0 v_1 \cdots v_k]
        :=
        \Bigl\{
           \sum_{i=0}^{k}\lambda_i v_i
           \Bigm|
           \lambda_i \ge 0,\;
           \sum_{i=0}^{k}\lambda_i = 1
        \Bigr\}.
\end{equation*}
The points $v_0,\dotsc,v_k$ are called the \emph{vertices} of $\sigma$.
Simplices determined by a subset of $v_0,\dotsc,v_k$ are called
\emph{faces} of $\sigma$.
Formally, a simplicial complex is a finite collection of simplices such
that every face of a simplex in the collection is also in the
collection, and the intersection of any two simplices is either empty or
a common face. 

\paragraph{Euler Characteristic}  
The \emph{Euler characteristic} $\chi$ is a topological invariant that
provides a summary of the ``shape'' or structure of a topological space,
such as a simplicial complex. It is defined as the alternating sum of the
number of simplices in each dimension, i.e.,
\begin{equation}
    \chi(K) = \sum_{k=0}^{d} (-1)^k \sigma_k(K),
\end{equation}
where $\sigma_k(K)$ is the number of $k$-dimensional simplices in
the simplicial complex $K$, and $d$ is the dimension of $K$.
As a topological invariant, the Euler characteristic remains unchanged
under transformations like homeomorphisms, making it---despite its
simplicity---a fundamental tool for distinguishing topological spaces.

\paragraph{Graph Neural Networks and Message Passing}  
\emph{Graph neural networks}~(GNNs) are a class of neural network
models designed to operate on graph-structured data. They extend
neural networks by incorporating the relational structure
inherent to graphs, enabling the learning of tasks such as \emph{node
classification}.
The core mechanism of many GNNs is \emph{message passing}, an
iterative procedure that propagates information through the graph to
update node representations based on their local neighborhood. Given
a graph $G = (V, E)$, where $V$ is the set of nodes and $E$ is the set
of edges, each node $v \in V$ is associated with a \emph{feature vector}
$\mathbf{x}_v$. At each layer $t$, i.e., each message-passing step,
a new node embedding $\mathbf{h}_v^{(t+1)}$ is calculated via
\begin{equation}
 \text{UPDATE}\left(\mathbf{h}_v^{(t)}, \text{AGG}\left(\{\mathbf{h}_u^{(t)} \mid u \in \mathcal{N}(v)\}\right)\right),
\end{equation}
where $\mathcal{N}(v)$ denotes the set of neighbors of node $v$, and
$\text{AGG}$ and $\text{UPDATE}$ are learnable functions parameterized
by the model. The $\text{AGG}$ function combines information from
neighboring nodes, while the $\text{UPDATE}$ function refines node
embeddings. Popular choices for these functions include mean, sum, and
attention mechanisms. Through multiple layers of message passing, GNNs
aggregate information from larger neighborhoods, capturing both local
and global graph structure~\citep{velickovic23a}.

\section{Related Work}

\emergencystretch=3em
\hyphenpenalty=10000

GNNs have revolutionized the field of graph representation learning by
enabling end-to-end learning of node/graph embeddings through message passing
\citep{kipf2016semi}. However, traditional GNNs face theoretical
limitations that pose fundamental obstructions to learning expressive
and general representations of graph data \citep{xu2018powerful}. Related to the
latter phenomenon, GNNs are known to suffer from issues like
\emph{oversmoothing}~\citep{zhang2023comprehensive,rusch2023survey} and
\emph{oversquashing}~\citep{di2023over}.
\citet{hamilton2017inductive} and \citet{velivckovic2017graph} have
addressed these issues by incorporating sampling and attention
mechanisms into the message-passing paradigm. However, even these
advancements often show limited performance, particularly in graphs with
high heterophily, and there is no ``general'' GNN capable of handling
both heterophilous and homophilous graphs.
Recent work in graph machine learning thus started incorporating
additional inductive biases into architectures, such as geometric
information~\citep{Southern23a, pei2020geom, joshiexpressive} or topological
information~\citep{horn2021topological, Verma24a}, with the ultimate
goal of improving the \emph{expressivity} of a model, i.e., its
capability to distinguish between non-isomorphic families of
graphs~\citep{Morris23a}. 

Many such endeavors arise from the field of
\emph{topological deep learning}~\citep{Papamarkou24a}, which aims to
develop models that are ``aware'' of the underlying
topology of a space, and thus also capable of handling data with
higher-order relations. 
Other constructions include special architectures for heterophilous
tasks that are \emph{not} based on
message-passing~\citep{lim2021large}, or \emph{modifications} of the
graph itself to improve predictive performance.
\citet{suresh2021breaking}, for instance, use edge rewiring to raise
graph assortativity and thus gain accuracy under low homophily,
whereas \citet{luan2022revisiting} mix feature channels during
aggregation to obtain state-of-the-art results on heterophilous
benchmarks.
Finally, as \citet{rampavsek2022recipe} show,
a hybrid model, combining message-passing~(local information) with
attention~(global information) via structural encodings,
may exhibit high expressivity and high scalability.
Subsequently, \citet{mueller2024attending} extended these results by
providing a taxonomy of elements related to the ``design space'' of
graph transformers.

As a geometrical-topological invariant, the $\ECT$ is poised to
contribute to more expressive architectures. Contributing   Being
already a popular tool in topological data analysis~\citep{Turner14a,
ghrist2018persistent}, recent extensions started tackling the
integration into deep-learning
architectures~\citep{roell2023differentiable} or the incorporation of
additional invariance properties~\citep{curry2022many,
marsh2024detecting}.
Despite advantageous performance in shape-classification tasks, however,
\emph{all} existing contributions solely focus on \emph{global} $\ECT$s
and do not discuss any \emph{local} aspects, which are crucial for our
approach.
In addition to the $\ECT$, some works also use other topology-based
tools in graph learning, primarily \emph{persistent homology}~(PH), an
expressive but computationally expensive geometrical-topological
invariant. Examples of this approach include \citet{Rieck19b} and
\citet{hofer2020graph}, who use PH for graph classification, or
\citet{zhao2019learning}, who learn a weighted kernel on topological
descriptors arising during PH computations.
Closest to our approach in spirit is \citet{zhao2020persistence}, who
include topological features of graph neighborhoods into a GNN, again
leveraging PH.
However, to the best of our knowledge, ours is the first work to develop
\emph{local} variants of the ECT for graph-learning tasks, analyze the
theoretical properties of such local variants, and finally show their
empirical utility for node classification.

\section{Methods}\label{sec:methods}

\paragraph{Euler Characteristic Transform~(ECT)}
The \emph{Euler Characteristic Transform}~(ECT) of a simplicial complex
$X \subset \mathbb{R}^n$ is a function $\ECT(X)\colon S^{n-1} \times \mathbb{R} \to \mathbb{Z}$, given by
\begin{equation}
    \ECT(X)(v,t) := \chi(\{ x \in X \mid x \cdot v \leq t \}),
\end{equation}
where $\chi$ denotes the Euler characteristic and $x \cdot v$ denotes
the Euclidean dot product.
The interpretation of $\ECT(X)$ is that it scans the ambient space of $X$ in every direction and records the Euler characteristic of the sublevel sets.
The $\ECT(X)$ is \emph{invertible}, meaning that $X$ can be recovered
from $\ECT(X)$, as long as $X$ is a so-called constructible
set~\citep{ghrist2018persistent, curry2022many}.
The main focus of this work are \emph{compact geometric simplicial
complexes}~(like geometric graphs), which are constructible, and
thus the invertibility theorem applies in our setting. 
Note that in practice, we \emph{approximate} $\ECT(X)$ via $\overline{\ECT}(X)_{(m,l)}
:=\ECT(X)_{\vert \{ v_1,\dots,v_m \} \times \{ t_1,\dots,t_l \} }$ for
uniformly-distributed directions $v_1,\dots,v_m \in S^{n-1}$ and
filtration steps $t_1, \dots , t_l \in \mathbb{R}$. Since $X$ is
compact, $t_1, \dots , t_l$ can be chosen to lie in a compact interval
$[a,b]$ with $t_1 = a$ and $t_l = b$, and so that the sequence $\{
t_i\}_i$ forms a uniform partition of $[a,b]$. We note that this
approximation is efficiently computable and has a natural representation
as a vector of dimension $m \cdot l$. Regarding the choice of the
magnitudes of $m,l$ we notice that the \emph{expected} nearest-neighbor
distance for uniform samples on $S^{n}$ scales as
$\mathcal{O}((\nicefrac{\log m}{m})^{1/n})$~\citep{beck1987irregularities},
and that the equidistant partitioning of a compact interval scales as
$\mathcal{O}(\nicefrac{1}{l})$, leading to $\mathcal{O}((\nicefrac{\log
m}{m})^{1/(n-1)}l^{-1})$ for the total approximation error of the domain of $\ECT(X)$.
\citet{curry2022many} prove that the aforementioned approximation
actually determines the \emph{true} value, provided that $m, l$ are
sufficiently large.
We notice that both translations and scalings of $X$ in the ambient
space lead to a reparametrization of $\ECT(X)$. Hence, $\ECT(X)$ remains
essentially unaltered~(up to a parameter change) under these two types
of transformations.

\paragraph{Local ECT ($\lECT$)}
Given a geometric simplicial complex $X \subset \mathbb{R}^n$ and
a vertex $x \in X$, we define the \emph{local ECT} of x with respect to
$k \geq 0$ as 
\begin{equation}
    \lECT_k(x;X) := \ECT(N_k(x;X)),
\end{equation}
where $N_k(x;X)$ denotes an appropriate \emph{local neighborhood} of $x$ in
$X$, whose locality scale is controlled by a parameter $k$. Usually,
$N_k(x;X)$ will be either the full subcomplex of $X$, which is spanned by
the $k$-hop neighbors of $x$, or the full subcomplex of $X$, which is
spanned by the $k$-nearest vertices of $x$. The first important special
case arises when $X$ is a $0$-dimensional simplicial complex, i.e.,
a point cloud. In this case, the full subcomplex of $X$, which is
spanned by the $k$-nearest vertices of $x$, $N_k(x;X)$, is given
by the $k$-nearest neighbors of $x$. 
Being based on the Euler Characteristic, the construction of $\lECT$s
appears to be purely topological at first glance. However, in light of
the invertibility theorem, we note that $\lECT(x;X)$ can be interpreted
as a \emph{fingerprint} of a local neighborhood of $x$ in $X$. The upshot is
that this fingerprint can be well approximated in practice, making it
possible to obtain \emph{local representations} of combinatorial data embedded
in Euclidean space.
Similar to the approximation of the ECT, this approximation works by
sampling $v_1,\dots,v_m \in S^{n-1}$ and $t_1, \dots , t_l \in
\mathbb{R}$, and considering $\overline{\ECT}(N_k(x;X))_{(m,l)}$,
instead of $\lECT_k(x;X)$. The latter quantity is well-computable in
practice, and the approximation error can be controlled by the sample
sizes $m$ and $l$, as we discussed above. Again, this approximation has
a natural representation as a vector of dimension $m \cdot l$, enabling
us to encode local structural information of point neighborhoods in an
approximate lossless way that can readily be used by machine-learning
algorithms for downstream tasks.

\subsection{Properties of $\lECT$s}

Our formulation of $\lECT$s provides a natural representation of local
neighborhoods of geometric simplicial complexes. One important special
case is that of \emph{featured graphs}, meaning graphs in which every node
admits a feature vector. The latter data structure forms the basis of
many modern graph-learning tasks, such as node classification, graph
classification, or graph regression. The predominant class of methods to
deal with these graph learning problems are message-passing graph neural
networks.
We develop an alternative procedure for dealing with featured graph
data, built on $\lECT$s and we show that $\lECT$s provide sufficient
information to perform message passing, which we explain in the
following.
\begin{definition}
A \emph{featured graph} is a (non-directed) graph $\mathcal{G}$ such that every node $v \in \mathcal{G}$ admits a feature vector $x(v) \in \mathbb{R}^n$. We denote the set of nodes of $\mathcal{G}$ by $V(\mathcal{G})$, and the set of edges by $E(\mathcal{G})$.
\end{definition}
We notice that a featured graph $\mathcal{G}$ can naturally be
interpreted as a graph embedded in $\mathbb{R}^n$, by representing each
node feature vector as a point in $\mathbb{R}^n$, and by drawing an edge
between two embedded points if and only if there is an edge between the
underlying nodes in $\mathcal{G}$. This construction yields a graph
isomorphism between $\mathcal{G}$ and the embedded graph if and only if
for any pair of nodes $v,w \in \mathcal{G}$ with $v \neq w$ we have
$x(v) \neq x(w)$ for their associated feature vectors. In practice, the
latter assumption can always be achieved by adding an arbitrarily small
portion of Gaussian noise to each feature vector, and we therefore may
restrict ourselves to featured graphs that yield an isomorphism on their
Euclidean embeddings.\footnote{%
Alternatively, we can drop the requirement of an embedding by noting
that a featured graph can be considered as an abstract simplicial
complex~$\mathcal{G}$ with an \emph{arbitrary} function~$f\colon \mathcal{G} \to \mathbb{R}^n$
defined on its vertices and edges. In this case, we may define the ECT as
$\ECT(X)(v,t) := \chi(\{ f^{-1}\{x \in \mathbb{R}^n \mid x \cdot v \leq
  t \})$. This formulation, developed by \citet{Marsh23a}, demonstrates
  that the ECT is \emph{generally} applicable and does not require node
  features to provide an embedding of a graph.
}
We now show that $\lECT$s are in fact expressive
graph-learning representations.
\begin{restatable}{theorem}{ThmExpressivity}
    Let $\mathcal{G}$ be a featured graph and let $\{ \lECT_1(x;\mathcal{G}) \}_{x}$ be the collection of local $\ECT$s with respect to the $1$-hop neighborhoods in $\mathcal{G}$. Then the collection $\{ \lECT_1(x;\mathcal{G}) \}_{x}$ provides the necessary~(non-learnable) information for performing a single message-passing step on $\mathcal{G}$, in the sense that for a given vertex $x \in \mathcal{G}$ one can reconstruct the feature vectors of its $1$-hop neighborhood from $\lECT_1(x;\mathcal{G})$.
  \label{thm:expressivity}
\end{restatable}
\cref{thm:expressivity} tells us that for a featured graph
$\mathcal{G}$, the collection $\{ \lECT_1(x;\mathcal{G}) \}_{x}$ already contains
sufficient information to perform a single step of message passing.
The advantage of using $\lECT$s instead of message passing to represent
featured graph data lies in the possibility to \emph{additionally} use $\{
\lECT_k(x;\mathcal{G}) \}_{x}$ for $k \geq 2$, which contain both
structural and feature vector information of larger neighborhoods of
nodes in the graph. This type of information is typically \emph{not}
explicitly available through message passing since passing messages to
non-direct neighbors depends on prior message passing steps, which
solely produce an aggregation of neighboring feature vectors.

In addition to essentially subsuming the information from one
message-passing step, we can also show that the $\lECT$ is ``aware'' of
local structures like subgraphs.
As shown by \citet{chen2020can}, message-passing graph neural networks
\emph{cannot} perform counting of induced subgraphs for \emph{any}
connected substructure consisting of 3 or more nodes.
By contrast, we will now show that $\ECT$s for featured graphs and their
local variants can indeed be used to perform subgraph counting. We start
with the definitions of the necessary concepts.
\begin{definition}\label{def:isomorphic_graphs}
    Two featured graphs $\mathcal{G}_1$ and $\mathcal{G}_2$ are
    isomorphic if there is a bijection $\pi\colon V(\mathcal{G}_1) \to
    V(\mathcal{G}_2)$, such that $(v,w) \in E(\mathcal{G}_1)$ if and
    only if $(\pi(v),\pi(w)) \in E(\mathcal{G}_2)$ and so that for all
    $v \in \mathcal{G}_1$ one has $x(v) = x(\pi(v))$ for the respective
    feature vectors.
\end{definition}

A featured graph $\mathcal{G}_S$ is called a \emph{subgraph} of $\mathcal{G}$
if $V(\mathcal{G}_S) \subset V(\mathcal{G})$ and $E(\mathcal{G}_S)
\subset E(\mathcal{G})$, such that the respective node features
remain unaltered under the induced embedding.
A featured graph $\mathcal{G}_S$ is called an \emph{induced subgraph} of
$\mathcal{G}$, if $\mathcal{G}_S$ is a subgraph of $\mathcal{G}$, and if
$E(\mathcal{G}_S) = E(\mathcal{G}) \cap \mathcal{G}_S$.
For two featured graphs $\mathcal{G}$ and $\mathcal{G}_S$, we define
$C_{\text{Sub}}(\mathcal{G};\mathcal{G}_S)$ to be the number of subgraphs in
$\mathcal{G}$ that are isomorphic to $\mathcal{G}_S$. Similarly, we
define $C_\text{Ind}(\mathcal{G};\mathcal{G}_S)$ to be the number of induced
subgraphs in $\mathcal{G}$ which are isomorphic to $\mathcal{G}_S$.
\begin{restatable}{theorem}{ThmIsomorphism}
    Two featured graphs $\mathcal{G}_1$ and $\mathcal{G}_2$ are isomorphic if and only if  $\ECT(\mathcal{G}_1)=\ECT(\mathcal{G}_2)$.
  \label{thm:isomorphic_graphs}
\end{restatable}
An immediate consequence of the previous Theorem is:
\begin{corollary}
$\ECT$s can perform subgraph counting.
\end{corollary}
We therefore conclude that $\ECT$-based methods for graph-representation
learning can be more powerful than message-passing-based approaches,
suggesting the development of \emph{hybrid} architectures, making use of
both message passing \emph{and} $\ECT$ variants.

\subsection{Rotation-Invariant Metric based on Local ECTs}\label{par:metric}

The aforementioned invariance properties of $\ECT$s with respect to
translations and scalings naturally raise the question if $\lECT$s may
be used to compare the local neighborhoods of two distinct
points/vertices.
However, the $\ECT$ is sensitive to rotations since rotating the
underlying simplicial complex leads to a misalignment of the respective
directions in $S^{n-1}$. Because a local comparison should \emph{not} depend on
the choice of a coordinate system, this property is a fundamental
obstruction of using $\lECT$ as a local similarity measure.
We therefore construct a novel \emph{rotation-invariant metric} as
follows. Let $X,Y \subset \mathbb{R}^n$ be two finite geometric simplicial
complexes. Since $X,Y$ are finite, $\ECT(X)$ and $\ECT(Y)$ only take
finitely many values, and we may define a similarity measure
$d_{\ECT}(X,Y)$  as
\begin{equation}
  d_{\ECT}(X,Y):=\!\!\!\inf_{\rho \in \text{SO}(n)}\!\!\norm{(\ECT(X)- \ECT(\rho Y))}_{\infty}\!.
\end{equation}
We first prove that this similarity measure satisfies the definitions of a metric.
\begin{restatable}{theorem}{ThmMetric}
    $d_{\ECT}$ is a metric on the collection of rotation classes of finite simplicial complexes embedded in $\mathbb{R}^n$.
  \label{thm:rotation_invariant_metric}
\end{restatable}

\cref{thm:rotation_invariant_metric} ensures that we may use $d_{\ECT}$
as a metric that measures the similarity between embedded simplicial
complexes up to rotation. In particular, for a simplicial complex $X
\subset \mathbb{R}^n$ and $x,y \in X$, we have a rotation-invariant
measure to compare local neighborhoods of $x$ and $y$ by defining
$d_{\ECT}^k(x,y;X)$ as
$\inf_{\rho \in \text{SO}(n)} \norm{\lECT_k(x;X)- \lECT_k( y;\rho X)}_{\infty}$.
In practice, we approximate $d_{\ECT}(X,Y)$ by 
\begin{equation}\label{metric_approx}
  \inf_{\rho \in \text{SO}(n)} \norm{\overline{\ECT}(X)_{(m,l)}- \overline{\ECT}(\rho Y)_{(m,l)}}_{\infty}
\end{equation}
for a choice of samples $v_1,\dots,v_m \in S^{n-1}$ and $t_1, \dots
, t_l \in \mathbb{R}$; this works analogously for the local version
$d_{\ECT}^k(x,y;X)$.
As discussed before, the approximations of the $\ECT$s used in
\cref{metric_approx} have a natural vector representation, so that
the $\norm{\bullet}_{\infty}$ in \cref{metric_approx} is in fact the
maximum of the entry-wise absolute differences between the two
respective representation vectors. Hence, the approximation shown in
\cref{metric_approx} is \emph{efficiently computable}. However, our
experiments in \cref{sec:Experiments} use the Euclidean metric for
differentiability reasons. 

\subsection{Limitations}

While $\lECT$s present clear advantages in preserving \emph{local}
details, there are some trade-offs to consider. In certain cases,
message-passing GNNs, which aggregate information across neighbors, may
be preferable, in particular for tasks where \emph{global} context is
more important than local details~(see \citet{Coupette25a} for a recent
analysis of graph-learning datasets under different perspectives).
Furthermore, while our method is
computationally feasible on medium-sized datasets~(as demonstrated in
our experiments), the complexity of ``na\"ively'' calculating $\lECT$s
increases for larger $k$ and with the size and density of the graph,
suggesting a need for improved methods~(see \cref{app:complexity} for an
extended discussion).

\section{Experiments}\label{sec:Experiments}

In this section, we present experiments to empirically evaluate the
performance of the $\lECT$-based approach in graph representation
learning, focusing on node-classification tasks. We aim to demonstrate
how $\lECT$ representations can capture structural information more
effectively than traditional message-passing mechanisms, especially in
scenarios with high heterophily~(even though we consider other scenarios
as well).
Our experiments compare the performance of $\lECT$-based models to
several standard GNN models, namely graph attention
networks~\citep[GAT]{velivckovic2017graph},
graph convolutional networks~\citep[GCN]{kipf2016semi}, graph
isomorphism networks~\citep[GIN]{xu2018powerful}, as well as
a heterophily-specific architecture~\citep[H2GCN]{zhu2020beyond}.
Furthermore, we showcase how the rotation-invariant metric from
\cref{sec:methods} may be used for spatial alignment of graph data.

\subsection{$\lECT$s in Graph Representation Learning}

The link between message passing and $\lECT$s~(cf.\
\cref{thm:expressivity}) encourages us to empirically validate the
expressivity of $\lECT$s for node-classification tasks.
Given a featured graph $\mathcal{G}$ and fixed $k \geq 0$, we assign
$\lECT_k(x;\mathcal{G})$ to every node $v \in \mathcal{G}$. We then use the $\lECT$
corresponding to a node together with the respective node feature vector
as the \emph{input} to classification models.
Subsequently, we focus on XGBoost~\citep{Chen16a}, as we found it to
outperform more complex models. However, our $\lECT$ can be used with
\emph{any} model.
Notice that our experiments are not about claiming
state-of-the-art performance but rather about showcasing that an
approach based on $\lECT$ yields results that are \emph{on a par with
and often superior} to more complex graph-learning techniques based on
message passing, while at the same time working well in \emph{both}
heterophilous and homophilous settings.

\paragraph{Implementation details}
%
We assume that we are given a featured graph
$\mathcal{G}$ such that there is an assignment $V(\mathcal{G}) \to
\mathcal{Y}$, with $V(\mathcal{G})$ being the node set of $\mathcal{G}$
and $\mathcal{Y}$ being the space of classes w.r.t.\ the underlying
node-classification task. For a fixed $k \geq 0$, $x \in V(\mathcal{G})$ and
$N_k(x;\mathcal{G})$ being the $k$-hop neighborhood of $x$ in
$\mathcal{G}$, we then approximate $\lECT_k(x;\mathcal{G})$ via
$\overline{\ECT}(N_k(x;\mathcal{G}))_{(m,l)}$ for sampled directions and
filtration steps, as explained in \cref{sec:methods}.
%
We use $m=l=64$~(but the number of samples may be tuned in practice) and
use the the resulting $m \cdot l$-dimensional vector(s)
$\overline{\ECT}(N_k(x;\mathcal{G}))_{(m,l)}$, together with the feature
vector of $x$, as additional inputs for the classifier.
The architecture of our baseline models includes a two-layer MLP after
every graph-neighborhood aggregation layer, as well as \emph{skip
connections} and \emph{layer normalization}.
We train each model for $1000$ epochs and report the test accuracy
corresponding to the state of the model that admits the \emph{maximum
validation accuracy} during training. This makes the predictive
performance of our baseline models directly comparable with
\citet{platonov2023critical}.

\paragraph{WebKB Datasets}
For all datasets of the WebKB collection~\cite{pei2020geom}, our $\lECT$-based
approach outperforms the baseline GNNs by
far~(cf.\ \cref{tab:ml_performance_webkb}; GraphSAGE results from
\citealt{Xu24a}).
While the combination of
both $\lECT_1$ and $\lECT_2$ performs best for ``Texas,'' using only
$\lECT_1$ leads to best performance for ``Wisconsin.'' However, for the
two aforementioned datasets, the combination of $\lECT_1$ and $\lECT_2$
only slightly improves the performance in comparison to $\lECT_1$,
suggesting that $1$-hop neighbor information is already sufficiently
informative here.

\begin{table}[tbp]
\centering
\setlength{\tabcolsep}{3pt} 
\let\b\bfseries
\small
\caption{%
  Performance (\emph{accuracy}, in percent) of graph-learning
  models on WebKB datasets~($5$ training runs).
}
\label{tab:ml_performance_webkb}
\begin{tabular}{@{}lSSS@{}}
\toprule
\textbf{Model} & {\textbf{Cornell}} & {\textbf{Wisconsin}} & {\textbf{Texas}} \\
\midrule
GCN               &  45.0 \pm 2.2  &  44.2 \pm 2.6  &  47.3 \pm 1.5  \\
GAT               &  44.7 \pm 2.9  &  48.2 \pm 2.0  &  51.7 \pm 3.2  \\
GIN               &  46.5 \pm 3.1  &  49.7 \pm 2.5  &  54.2 \pm 2.9  \\
GraphSAGE         &\b76.0 \pm 3.5  &  72.9 \pm 1.9  &  71.8 \pm 2.4  \\
H2GCN             &  66.2 \pm 3.5  &  70.2 \pm 2.3  &  72.3 \pm 3.0  \\
\midrule
$\lECT_1$         &  66.8 \pm 4.2  &\b81.2 \pm 2.9  &  74.6 \pm 0.5  \\
$\lECT_2$         &  67.0 \pm 4.9  &  76.1 \pm 2.8  &  73.8 \pm 2.6  \\
$\lECT_1+\lECT_2$ &  67.1 \pm 4.1  &  78.5 \pm 2.6  &\b74.8 \pm 3.1  \\
\bottomrule
\end{tabular}
\end{table}

\paragraph{Heterophilous Datasets}
\citet{platonov2023critical} introduced several \emph{heterophilous}
datasets; we validate our method on ``Amazon Ratings'' and ``Roman
Empire,'' again observing that the combination of $\lECT_1 + \lECT_2$
performs best, substantially outperforming baseline models~(cf.\
\cref{tab:ml_performance_heterophilous}).
The results are closely aligned with findings by
\citet{platonov2023critical}, i.e.,  that specialized architectures like
H2GCN often perform less well than ``standard'' architectures.
Moreover, $\lECT_1$ outperforms $\lECT_2$ on ``Roman Empire,'' while
$\lECT_2$ outperforms $\lECT_1$ on ``Amazon Ratings.''
We interpret this as $1$-hop neighborhoods being particularly
informative for ``Roman Empire,'' while $2$-hop neighborhoods are more
informative for ``Amazon Ratings.''

\begin{table}[tbp]
\centering
\setlength{\tabcolsep}{3pt} 
\let\b\bfseries
\small
\caption{
  Performance (\emph{accuracy}, in percent) of graph-learning
  models on \emph{heterophilous} datasets~($5$ training runs).
}
\label{tab:ml_performance_heterophilous}
\begin{tabular}{@{}lSSS@{}}
\toprule
\textbf{Model}      & \textbf{Amazon Ratings} &  \textbf{Roman Empire}\\
\midrule
GCN                 &  42.3 \pm 0.7           &   73.3 \pm 0.8        \\
GAT                 &  44.6 \pm 0.9           &   76.4 \pm 1.2        \\
GIN                 &  44.1 \pm 0.8           &   56.8 \pm 1.0        \\
GraphSAGE           &  42.2 \pm 0.5           &   77.4 \pm 0.8        \\
H2GCN               &  40.1 \pm 0.7           &   64.2 \pm 0.9        \\
\midrule
$\lECT_1$           &  48.4 \pm 0.3           &   80.4 \pm 0.4        \\
$\lECT_2$           &  49.6 \pm 0.3           &   78.0 \pm 0.3        \\
$\lECT_1+\lECT_2$   &\b49.8 \pm 0.3           & \b81.1 \pm 0.4        \\
\bottomrule
\end{tabular}
\end{table}

\begin{table}[tbp]
\centering%
\setlength{\tabcolsep}{3pt}
\let\b\bfseries%
\caption{%
  Performance (\emph{accuracy}, in percent) of graph-learning models on
  Amazon datasets~($5$ training runs).%
}
\label{tab:ml_performance_computers_photo}
\small%
\begin{tabular}{@{}lSSS@{}}
\toprule
\textbf{Model}    & \textbf{Computers} & \textbf{Photo}  \\
\midrule
GCN               &   91.6 \pm 1.6  & 93.6 \pm 1.7  \\
GAT               &\b 92.4 \pm 1.3  & 94.8 \pm 1.1 \\
GIN               &   55.9 \pm 1.5  & 82.2 \pm 1.3  \\
GraphSAGE         &   89.2 \pm 0.9  & 92.5 \pm 0.7  \\
H2GCN             &   84.5 \pm 1.4  & 92.8 \pm 1.2  \\
\midrule
$\lECT_1$ & 89.6 \pm 0.3  & 94.1 \pm 0.3  \\
$\lECT_2$ & 90.1 \pm 0.5  & 94.4 \pm 0.7  \\
$\lECT_1+\lECT_2$ & 92.2 \pm 0.6  & \b 94.9 \pm 0.6 \\
\bottomrule
\end{tabular}
\end{table}

\paragraph{Amazon dataset}
The Amazon dataset~\citep{shchur2018pitfalls} consists of the two
co-purchase graphs ``Computers'' and ``Photo.''
While GAT outperforms all methods on ``Computers,'' the combination of
$\lECT_1$ and $\lECT_2$ performs best on ``Photo''~(cf.\
\cref{tab:ml_performance_computers_photo}). Overall, $\lECT$-based
methods exhibit competitive performance here, given that they are
\emph{not} based on message passing.

\begin{table}[btp]
\centering%
\small%
\setlength{\tabcolsep}{3pt}
\let\b\bfseries%
\caption{%
  Performance (\emph{accuracy}, in percent) of graph-learning
  models on \emph{heterophilous} datasets~($5$ training runs).%
}
\label{tab:ml_performance_merged}%
\begin{tabular}{@{}lSSS@{}}
\toprule
\textbf{Model}    & \textbf{Actor} & \textbf{Squirrel} & \textbf{Chameleon} \\
\midrule
GCN               &   30.7 \pm 2.1  &   28.9 \pm 1.4  &   42.8 \pm 1.8      \\
GAT               &   31.1 \pm 1.8  &   31.8 \pm 1.3  &   47.3 \pm 1.3      \\
GIN               &   26.5 \pm 2.0  &   35.4 \pm 1.5  &   43.1 \pm 1.7      \\
GraphSAGE         &   30.2 \pm 1.4  &   33.3 \pm 0.7  &   45.2 \pm 1.3      \\
H2GCN             &   30.7 \pm 1.9  &\b 40.8 \pm 1.4  &\b 62.7 \pm 1.6      \\
\midrule
$\lECT_1$         &\b 31.4 \pm 1.9  &   35.6 \pm 0.7  &   43.5 \pm 1.7     \\
$\lECT_2$         &   30.1 \pm 1.3  &   35.6 \pm 0.8  &   40.4 \pm 1.5     \\
$\lECT_1+\lECT_2$ &   30.9 \pm 0.7  &   35.3 \pm 1.5  &   43.9 \pm 0.7     \\
\bottomrule
\end{tabular}
\end{table}

\paragraph{Actor/Wikipedia Datasets}
%
Moving to additional heterophilous datasets with high feature
dimensionality, we compare predictive performance on
``Actor''~\citep{pei2020geom} as well as ``Chameleon'' and
``Squirrel''~\citep{rozemberczki2021multi}; cf.\
\cref{tab:ml_performance_merged}.
For ``Actor'', the $\lECT_1$ model achieves the highest accuracy, while
$\lECT_1 + \lECT_2$ performs slightly worse. $\lECT_2$ performs the
lowest on this dataset, suggesting that larger neighborhoods are
detrimental here.
For the other datasets, the heterophily-specific
model H2GCN performs best. However, $\lECT_1$ and
$\lECT_1$ and $\lECT_2$ exhibit similar~(or even better) performance as
\emph{all other standard baselines}, showing the advantages of $\lECT$ methods
even in the absence of hyperparameter tuning.

\begin{table}[btp]
\centering
\small
\setlength{\tabcolsep}{3pt}
\let\b\bfseries%
\caption{%
  Performance (\emph{accuracy}, in percent) of graph-learning models on
  ``Planetoid'' datasets~($5$ training runs).
}
\label{tab:ml_performance_planetoid}
\begin{tabular}{@{}lSSS@{}}
\toprule
\textbf{Model}    & \textbf{Cora} & \textbf{CiteSeer} & \textbf{PubMed} \\
\midrule
GCN               &  88.1 \pm 1.2 &  74.6 \pm 1.5 &  85.3 \pm 4.7 \\
GAT               &\b88.3 \pm 1.1 &\b75.3 \pm 1.5 &  85.7 \pm 4.2 \\
GIN               &  85.0 \pm 1.5 &  72.2 \pm 1.7 &  87.0 \pm 0.5 \\
GraphSAGE         &  82.2 \pm 1.2 &  68.1 \pm 1.2 &  84.3 \pm 0.7 \\
H2GCN             &  85.4 \pm 1.6 &  72.4 \pm 1.9 &  86.4 \pm 0.5 \\
\midrule
$\lECT_1$         &  87.6 \pm 0.6 &  72.1 \pm 0.6 &  90.2 \pm 0.5 \\
$\lECT_2$         &  87.2 \pm 0.7 &  72.3 \pm 0.8 &\b90.3 \pm 0.5 \\
$\lECT_1+\lECT_2$ &  87.8 \pm 0.6 &  72.5 \pm 0.7 &\b90.3 \pm 0.5 \\
\bottomrule
\end{tabular}
\end{table}

\paragraph{Planetoid Datasets}
%
We also analyze node-classification performance on datasets from the
``Planetoid'' collection~\citep{yang2016revisiting}, comprising
``Cora,'' ``CiteSeer,'' and ``PubMed.''
We trained all models using
a random 75--25 split; cf.\ \cref{tab:ml_performance_planetoid}.
Although GCN and GAT perform slightly better than $\lECT$ methods for
``Cora'' and ``CiteSeer,'' the gap is surprisingly small. For
``PubMed,'' the $\lECT$-based models even outperform both \emph{all
comparison partners}.
These findings suggest that the \emph{theoretical} expressivity of
$\lECT$s, which we formally established in \cref{sec:methods}, also has
\emph{practical} implications, providing an alternative way of dealing
with graph data that is not restricted by the underlying model
architecture and therefore allows for interpretability.

\paragraph{Post-hoc Evaluation}
%
To evaluate our methods \emph{across} datasets, we used critical
difference diagrams, enabling us to compare the performance of various
models on both homophilic and heterophilic graph datasets~(cf.\
\Cref{app:evaluation}).
The results highlight the superior performance of $\lECT$-based
approaches over standard baselines and heterophily-specific
architectures such as H2GCN. Notably, the $\lECT_1 + \lECT_2$ method
achieved the best average rank of 2 and even the least effective
$\lECT$-based model ($\lECT_2$) outperformed all non-$\lECT$-based
methods, including GAT. Further evaluation against heterophily-specific
models reported in the literature corroborates these findings. In
comparison to state-of-the-art methods such as H2GCN, GPR-GNN, and
GloGNN, $\lECT_1 + \lECT_2$ achieved a competitive rank, matching GloGNN
and surpassing both GAT and GT. Despite being a \emph{general-purpose
approach} not specifically designed for heterophilic graphs,
$\lECT$-based methods demonstrated exceptional adaptability and
robustness across diverse graph structures. These results establish
$\lECT$-based architectures as a versatile and high-performing solution
for node classification tasks, suitable for tackling challenges across
a wide range of graph data.

\paragraph{Summary of Node-Classification Experiments}
%
We find that $\lECT$s work particularly well in situations where
aggregating neighboring information is \emph{inappropriate}, such as
when dealing with graphs that
exhibit a high degree of heterophily.
In such contexts, our approach may outperform message-passing-based
methods.
The upshot of our method is that local graph information can be
incorporated \emph{without} the architectural necessity to diffuse information
along the graph structure, as it is the case for message-passing-based
models.
While this discrete diffusion process induced by message passing
is useful for a plethora of graph-learning tasks, it can also be an
obstruction in learning the right representation for tasks where node
features of neighbors in the graph should not be aggregated~(cf.\
\citealt{Coupette25a} for a recent analysis of graph-learning datasets
in the context of message passing).
In this sense, $\lECT$s naturally overcome a \emph{fundamental
limitation} inherent to message-passing methods. Another advantage of
$\lECT$s is that they are \emph{agnostic} to the choice of the downstream
model. This permits us to use models that are easy to tune, enabling
practitioners to make use of their graph data without necessarily having
specialized knowledge in GNN training and tuning while at the same time
also working well in the small-sample regime.
Moreover, it permits using models that are \emph{interpretable}, making
our method well-suited for domains where regulatory demands often ask
for levels of interpretability that cannot readily be achieved
by~(graph) neural networks. In fact, by using feature importance
values~(which are directly available for tree-based algorithms like
XGBoost) and since the entries of the $\lECT$ vectors that are used as
the input for the model can be linked to the directions in the
calculation of the $\lECT$s, one may obtain a deeper understanding
of \emph{how} the model arrives at predictions~(see \Cref{app:ablation}
for a more in-depth discussion and an ablation on the number of
directions used to calculate $\lECT$s).

\subsection{Learning Spatial Alignment of Geometric Graphs}

In the following, we use the approach described in
\Cref{sec:methods} in order to learn the spatial alignment of two
data spaces by re-rotating one into the other. We start by showing that
synthetic data, which only differs up to a rotation, can be re-aligned
using $\lECT$s. Moreover, we show that this alignment is stable
with respect to noise, making it a robust measure for the comparison of
local neighborhoods in data. 
In comparison to other spatial alignment methods like the iterative
closest point algorithm, ours does \emph{not} necessitate the
computation of all pairwise distances between points in the respective
spaces. The latter is often a computational bottleneck, especially for
large datasets, thus positioning our method for spatial alignment as
a computationally more efficient method in practice. 
While alignment methods like \emph{Procrustes alignment} are restricted to
point-cloud data, we observe that our approach is also capable of
aligning embedded graph data~(or, more generally, simplicial complexes).
This makes it particularly useful for dealing with \emph{geometric
graphs}, constituting a highly-efficient alternative to more involved
machine-learning models like geometric GNNs~\citep{joshiexpressive}. 

\begin{figure*}[tb]
    \centering
    \subcaptionbox{Examples of geometric graphs\label{sfig:k_star_graph}}{%
      \includegraphics[width=0.25\linewidth]{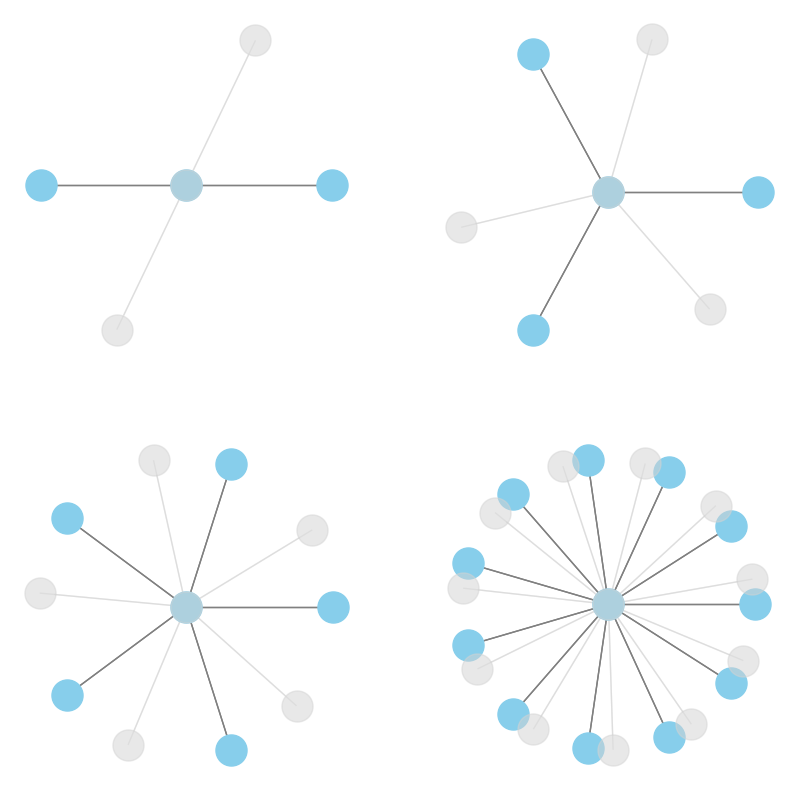}
    }%
    \quad%
    \subcaptionbox{Alignment results}{%
  \begin{tikzpicture}
    \pgfplotsset{%
      boxplot/draw direction = x,
      boxplot/whisker range = 10000, 
    }
    \begin{axis}[%
      axis x line*     = bottom,
      axis y line*     = none,
      ytick            = {1, 2, 3, 4},
      yticklabels      = {11-star graph, 5-star graph, 3-star graph, 2-star graph (line)},
      height           = 4.0cm,
      width            = 0.8\linewidth,
      mark size        = 1pt,
      tick align       = outside,
      xlabel           = {Hausdorff distance},
      ylabel           = {},
      xticklabel style = {%
        /pgf/number format/fixed,
        /pgf/number format/precision=2,
      },
    ]
      \addplot[boxplot] table[y index = 0] {Data/aligned_losses_11_star_graph.txt};
      \addplot[boxplot] table[y index = 0] {Data/aligned_losses_5_star_graph.txt};
      \addplot[boxplot] table[y index = 0] {Data/aligned_losses_3_star_graph.txt};
      \addplot[boxplot] table[y index = 0] {Data/aligned_losses_2_star_graph.txt};

      \addplot [mark=*, mark size=3, only marks, forget plot] coordinates {(0.121,1) (0.315,2) (0.335,3) (0.515,4)};
    \end{axis}
  \end{tikzpicture}%
}
    \caption{%
      A comparison of the Hausdorff distances of aligned graphs. The
      black dots represents the Hausdorff distance between the original
      graph and a randomly-rotated version of itself.
      Our $\lECT$-based alignment always results in substantially
      lower distances, with a median distance close to zero.
}
    \label{fig:graph_alignment}
\end{figure*}

For the subsequent learning problem~(see \Cref{app:Spatial Alignment}
for an example of how to align point-cloud data), we assume that we
are given two embedded simplicial complexes $X,Y \subset \mathbb{R}^n$.
In light of \Cref{sec:methods}, the metric properties of $d_{\ECT}$
ensure that $d_{\ECT}(X,Y)=0$ if $X$ and $Y$ only differ up to
a rotation. We therefore approximate $d_{\ECT}(X,Y)$ via
\begin{equation}\label{eq:learning_metric}
  \min_{\rho \in \mathrm{SO}(n)} \norm{\overline{\ECT}(X)_{(m,l)}- \overline{\ECT}(\rho Y)_{(m,l)}}_{2}^2
\end{equation}
for a choice of directions $v_1,\dots,v_m \in S^{n-1}$ and filtration
steps $t_1, \dotsc, t_l \in \mathbb{R}$.
As explained in \Cref{sec:methods}, the $\ECT$ approximations are given
by vectors, making it feasible to approach the above learning problem by
any gradient-based learning algorithm. The advantage of
this formulation is that it yields both the rotation-invariant loss and
the rotation that leads to this minimum loss.

Geometric graphs provide a compelling example of the utility of our
method, particularly in addressing the challenging problem of graph
re-alignment. For this analysis, we focus on a specific type of
geometric graph known as the $k$-star graph. A $k$-star graph is defined
as a tree with one internal node and $k$ leaves, i.e., a simple
structure with relevant geometric properties~(cf.\ \Cref{sfig:k_star_graph}).
To embed such graphs into a 2D space, we assign a unique 2D vector to each
node, ensuring the assigned vectors are equidistant to maintain
structural symmetry. Furthermore, to introduce variability and assess
robustness, we subject each embedded graph to a random 2D rotation, simulating
realistic perturbations encountered in practical settings. The central
goal is to recover the original graph's orientation by learning the
rotation matrix using the metric defined \Cref{sec:methods}.
We evaluate the performance of our approach by measuring the similarity
between the original graph and its re-rotated version using the
\emph{Hausdorff distance}, a metric that quantifies the maximal deviation
between two sets of points. To ensure significance, we repeated the
learning procedure 200 times, maintaining consistent initializations for
both the graph and the rotation matrix. \Cref{fig:graph_alignment} shows
the results; we observe that our realignment procedure
\emph{consistently} achieves small Hausdorff distances with medians near
zero, indicating the successful recovery of the original graph's
orientation. By contrast, the Hausdorff distances between the original
graph and its perturbed version are significantly larger.
While equivariant GNNs have been shown to struggle with distinguishing
the orientation of rotationally-symmetric
structures~\citep{joshiexpressive}, our method generates graph
representations that are initially sensitive to rotations but can be made
rotation-invariant through alignment of the underlying $\ECT$s. 

These
findings highlight the potential of $\ECT$-based metrics for \emph{robust}
geometric graph alignment, paving the way for broader applications in
domains requiring precise graph-based comparisons such as the analysis
of geometric graphs with constrained parameter budgets~\citep{Maggs24a}.
This result is particularly notable since the $\lECT$ easily outperforms
more involved architectures, pointing towards its overall utility as an
alternative to message-passing graph neural networks. At the same time, we believe that the $\lECT$ could also help in aligning higher-order data like geometric simplicial complexes.

\section{Discussion}

We introduced the \emph{Local Euler Characteristic Transform}~($\lECT$),
providing a novel approach to graph-representation learning
that preserves local structural information \emph{without} relying on
aggregation.
Our method addresses fundamental limitations in message-passing neural
networks, particularly in tasks where aggregating neighboring
information is suboptimal, such as in graphs with heterophily. By
retaining critical local details, $\lECT$s enable more nuanced and
expressive representations, offering significant advantages in node
classification tasks.
One key strength of our approach is its model-agnostic nature, allowing
it to be paired with interpretable machine learning models~(in our
experiments, an XGBoost model was used to provide feature importance
values, for instance).
This is particularly useful in domains such as healthcare, finance, and
legal applications, where regulatory frameworks demand high levels of
transparency and interpretability that are often difficult to achieve
with black-box neural networks. By leveraging $\lECT$s, we can satisfy
these requirements while maintaining the high expressivity and high
predictive performance required for graph-learning tasks.
In this way, $\lECT$-based representations offer a novel pathway toward interpretable machine learning on graph data: they yield topologically-grounded, vectorized encodings of local structure that not only retain predictive power but also support downstream diagnostics.

\paragraph{Future Work: Higher-Order Domains }Being situated at the intersection of geometry and topology, our $\lECT$
method is part of the nascent field of \emph{topological deep
learning}~\citep{Papamarkou24a}, which aims to develop novel inductive
biases that are capable of leveraging additional structural information
from data, both in the context of graphs and in the context of
higher-order domains like \emph{simplicial complexes}. It is in this
context where we believe that future work could be beneficial, in
particular since recent research~\citep{Ballester25a} showed that tasks
on such domains are highly challenging for existing GNNs. Given the
advantageous scalability properties of the $\ECT$~\citep{roell2023differentiable, Turner14a}, we believe that this constitutes
a useful avenue for future research.

\paragraph{Future Work: Comparing Representations}
Containing both geometrical and topological components, we also believe $\ECT$-based methods to be suitable in navigating different representations. Since the $\ECT$ can be considered a compression technique with controllable fidelity properties~\citep{Roell25a}, it could be useful in condensing latent spaces, thus permitting simple and efficient comparisons of models as hyperparameters are being varied. Such \emph{multiverse analyses} are vital for ensuring reproducibility~\citep{Bell22a, Wayland24a}.

\paragraph{Future Work: Hybrid Models}
Beyond representation learning on graphs, our $\lECT$ framework also
opens up new applications in domains where local structure is
critical, such point-cloud analysis~(including sensor data or computer-graphics data), 3D shape analysis, or data from the
life sciences~(like molecular data or biological networks).
Future work could thus explore more efficient algorithms for computing
$\ECT$s and $\lECT$s at scale, as well as hybrid approaches that balance
local and global information more effectively. A highly-relevant
direction would be  the integration of $\lECT$s into existing
message-passing neural networks, similar to recent work that leverages
persistent homology~\citep{Verma24a}.
Moreover, heterophily-specific mechanisms such as a separation of
neighborhood aggregation~(as used in specialized GNN architectures) may
be incorporated into our $\lECT$-based framework to further strengthen
its expressivity in the presence of high-heterophily graphs.

\section*{Software and Data}

Our code is available under \url{https://github.com/aidos-lab/Diss-l-ECT}.
We make use of standard benchmarking datasets, loaded and processed via 
the \texttt{PyTorch Geometric} library~\citep{Fey19a}.

\section*{Acknowledgments}

The first author acknowledges the use of ChatGPT for grammar
suggestions.
Both authors are very grateful for the discussions with the anonymous
reviewers, in particular reviewer \texttt{nYTM}, and the area chair,
who also believed in the merits of this work.
This work has received funding from the Swiss State Secretariat for
Education, Research, and Innovation~(SERI).

\section*{Impact Statement}

This paper presents work whose goal is to advance the field of graph
machine learning. There are many potential societal consequences of our
work, none which we feel must be specifically highlighted here.

\bibliography{main}

\begin{thebibliography}{61}
\providecommand{\natexlab}[1]{#1}
\providecommand{\url}[1]{\texttt{#1}}
\expandafter\ifx\csname urlstyle\endcsname\relax
  \providecommand{\doi}[1]{doi: #1}\else
  \providecommand{\doi}{doi: \begingroup \urlstyle{rm}\Url}\fi

\bibitem[Ballester et~al.(2025)Ballester, Röell, Schmid, Alain, Escalera,
  Casacuberta, and Rieck]{Ballester25a}
Ballester, R., Röell, E., Schmid, D.~B., Alain, M., Escalera, S., Casacuberta,
  C., and Rieck, B.
\newblock {MANTRA}: The {M}anifold {T}riangulations {A}ssemblage.
\newblock In \emph{International Conference on Learning Representations}, 2025.
\newblock URL \url{https://openreview.net/forum?id=X6y5CC44HM}.

\bibitem[Beck(1987)]{beck1987irregularities}
Beck, J.
\newblock {Irregularities of distribution. I}.
\newblock \emph{Acta Mathematica}, 159:\penalty0 1--49, 1987.

\bibitem[Bell et~al.(2022)Bell, Kampman, Dodge, and Lawrence]{Bell22a}
Bell, S.~J., Kampman, O., Dodge, J., and Lawrence, N.
\newblock Modeling the machine learning multiverse.
\newblock In Koyejo, S., Mohamed, S., Agarwal, A., Belgrave, D., Cho, K., and
  Oh, A. (eds.), \emph{Advances in Neural Information Processing Systems},
  volume~35, pp.\  18416--18429. Curran Associates, Inc., 2022.

\bibitem[Bo et~al.(2021)Bo, Wang, Shi, and Shen]{bo2021beyond}
Bo, D., Wang, X., Shi, C., and Shen, H.
\newblock Beyond low-frequency information in graph convolutional networks.
\newblock In \emph{Proceedings of the AAAI Conference on Artificial
  Intelligence}, volume~35, pp.\  3950--3957, 2021.

\bibitem[Borgwardt et~al.(2020)Borgwardt, Ghisu, Llinares-L{\'o}pez, O'Bray,
  and Rieck]{Borgwardt20}
Borgwardt, K., Ghisu, E., Llinares-L{\'o}pez, F., O'Bray, L., and Rieck, B.
\newblock Graph kernels: State-of-the-art and future challenges.
\newblock \emph{Foundations and Trends® in Machine Learning}, 13\penalty0
  (5--6):\penalty0 531--712, 2020.

\bibitem[Chen \& Guestrin(2016)Chen and Guestrin]{Chen16a}
Chen, T. and Guestrin, C.
\newblock {XGBoost}: A scalable tree boosting system.
\newblock In \emph{Proceedings of the 22nd ACM SIGKDD International Conference
  on Knowledge Discovery \& Data Mining}, pp.\  785--794, 2016.

\bibitem[Chen et~al.(2020)Chen, Chen, Villar, and Bruna]{chen2020can}
Chen, Z., Chen, L., Villar, S., and Bruna, J.
\newblock Can graph neural networks count substructures?
\newblock In Larochelle, H., Ranzato, M., Hadsell, R., Balcan, M., and Lin, H.
  (eds.), \emph{Advances in Neural Information Processing Systems}, volume~33,
  pp.\  10383--10395. Curran Associates, Inc., 2020.

\bibitem[Chien et~al.(2021)Chien, Peng, Li, and Milenkovic]{chien2020adaptive}
Chien, E., Peng, J., Li, P., and Milenkovic, O.
\newblock Adaptive universal generalized {PageRank} graph neural network.
\newblock In \emph{International Conference on Learning Representations}, 2021.
\newblock URL \url{https://openreview.net/forum?id=n6jl7fLxrP}.

\bibitem[Coupette et~al.(2025)Coupette, Wayland, Simons, and
  Rieck]{Coupette25a}
Coupette, C., Wayland, J., Simons, E., and Rieck, B.
\newblock No metric to rule them all: Toward principled evaluations of
  graph-learning datasets.
\newblock \emph{arXiv:2502.02379}, 2025.

\bibitem[Curry et~al.(2022)Curry, Mukherjee, and Turner]{curry2022many}
Curry, J., Mukherjee, S., and Turner, K.
\newblock How many directions determine a shape and other sufficiency results
  for two topological transforms.
\newblock \emph{Transactions of the American Mathematical Society, Series B},
  9\penalty0 (32):\penalty0 1006--1043, 2022.

\bibitem[Di~Giovanni et~al.(2023)Di~Giovanni, Giusti, Barbero, Luise, Lio, and
  Bronstein]{di2023over}
Di~Giovanni, F., Giusti, L., Barbero, F., Luise, G., Lio, P., and Bronstein,
  M.~M.
\newblock On over-squashing in message passing neural networks: The impact of
  width, depth, and topology.
\newblock In Krause, A., Brunskill, E., Cho, K., Engelhardt, B., Sabato, S.,
  and Scarlett, J. (eds.), \emph{Proceedings of the 40th International
  Conference on Machine Learning}, volume 202 of \emph{Proceedings of Machine
  Learning Research}, pp.\  7865--7885. PMLR, 2023.

\bibitem[Du et~al.(2022)Du, Shi, Fu, Ma, Liu, Han, and Zhang]{du2022gbk}
Du, L., Shi, X., Fu, Q., Ma, X., Liu, H., Han, S., and Zhang, D.
\newblock {GBK-GNN}: Gated bi-kernel graph neural networks for modeling both
  homophily and heterophily.
\newblock In \emph{Proceedings of the ACM Web Conference 2022}, pp.\
  1550--1558, 2022.

\bibitem[Fey \& Lenssen(2019)Fey and Lenssen]{Fey19a}
Fey, M. and Lenssen, J.~E.
\newblock Fast graph representation learning with {PyTorch Geometric}.
\newblock In \emph{ICLR Workshop on Representation Learning on Graphs and
  Manifolds}, 2019.

\bibitem[Ghrist et~al.(2018)Ghrist, Levanger, and Mai]{ghrist2018persistent}
Ghrist, R., Levanger, R., and Mai, H.
\newblock Persistent homology and {E}uler integral transforms.
\newblock \emph{Journal of Applied and Computational Topology}, 2:\penalty0
  55--60, 2018.

\bibitem[Hamilton et~al.(2017)Hamilton, Ying, and
  Leskovec]{hamilton2017inductive}
Hamilton, W., Ying, Z., and Leskovec, J.
\newblock Inductive representation learning on large graphs.
\newblock In Guyon, I., Luxburg, U.~V., Bengio, S., Wallach, H., Fergus, R.,
  Vishwanathan, S., and Garnett, R. (eds.), \emph{Advances in Neural
  Information Processing Systems}, volume~30. Curran Associates, Inc., 2017.

\bibitem[He et~al.(2016)He, Zhang, Ren, and Sun]{he2016deep}
He, K., Zhang, X., Ren, S., and Sun, J.
\newblock Deep residual learning for image recognition.
\newblock In \emph{Proceedings of the IEEE Conference on Computer Vision and
  Pattern Recognition}, pp.\  770--778, 2016.

\bibitem[Hofer et~al.(2020)Hofer, Graf, Rieck, Niethammer, and
  Kwitt]{hofer2020graph}
Hofer, C.~D., Graf, F., Rieck, B., Niethammer, M., and Kwitt, R.
\newblock Graph filtration learning.
\newblock In Daumé~III, H. and Singh, A. (eds.), \emph{Proceedings of the 37th
  International Conference on Machine Learning}, volume 119 of
  \emph{Proceedings of Machine Learning Research}, pp.\  4314--4323. PMLR,
  2020.

\bibitem[Horn et~al.(2022)Horn, {De Brouwer}, Moor, Moreau, Rieck, and
  Borgwardt]{horn2021topological}
Horn, M., {De Brouwer}, E., Moor, M., Moreau, Y., Rieck, B., and Borgwardt, K.
\newblock Topological graph neural networks.
\newblock In \emph{International Conference on Learning Representations}, 2022.
\newblock URL \url{https://openreview.net/forum?id=oxxUMeFwEHd}.

\bibitem[Joshi et~al.(2023)Joshi, Bodnar, Mathis, Cohen, and
  Li{\`o}]{joshiexpressive}
Joshi, C.~K., Bodnar, C., Mathis, S.~V., Cohen, T., and Li{\`o}, P.
\newblock On the expressive power of geometric graph neural networks.
\newblock In \emph{International Conference on Learning Representations}, 2023.
\newblock URL \url{https://openreview.net/forum?id=Rkxj1GXn9_}.

\bibitem[Kipf \& Welling(2017)Kipf and Welling]{kipf2016semi}
Kipf, T.~N. and Welling, M.
\newblock Semi-supervised classification with graph convolutional networks.
\newblock In \emph{International Conference on Learning Representations}, 2017.
\newblock URL \url{https://openreview.net/forum?id=SJU4ayYgl}.

\bibitem[Li et~al.(2022)Li, Zhu, Cheng, Shan, Luo, Li, and Qian]{li2022finding}
Li, X., Zhu, R., Cheng, Y., Shan, C., Luo, S., Li, D., and Qian, W.
\newblock Finding global homophily in graph neural networks when meeting
  heterophily.
\newblock In Chaudhuri, K., Jegelka, S., Song, L., Szepesvari, C., Niu, G., and
  Sabato, S. (eds.), \emph{Proceedings of the 39th International Conference on
  Machine Learning}, volume 162 of \emph{Proceedings of Machine Learning
  Research}, pp.\  13242--13256. PMLR, 17--23 Jul 2022.

\bibitem[Lim et~al.(2021)Lim, Hohne, Li, Huang, Gupta, Bhalerao, and
  Lim]{lim2021large}
Lim, D., Hohne, F., Li, X., Huang, S.~L., Gupta, V., Bhalerao, O., and Lim,
  S.~N.
\newblock Large scale learning on non-homophilous graphs: New benchmarks and
  strong simple methods.
\newblock In Ranzato, M., Beygelzimer, A., Dauphin, Y., Liang, P., and Vaughan,
  J.~W. (eds.), \emph{Advances in Neural Information Processing Systems},
  volume~34, pp.\  20887--20902. Curran Associates, Inc., 2021.

\bibitem[Luan et~al.(2022)Luan, Hua, Lu, Zhu, Zhao, Zhang, Chang, and
  Precup]{luan2022revisiting}
Luan, S., Hua, C., Lu, Q., Zhu, J., Zhao, M., Zhang, S., Chang, X.-W., and
  Precup, D.
\newblock Revisiting heterophily for graph neural networks.
\newblock In Koyejo, S., Mohamed, S., Agarwal, A., Belgrave, D., Cho, K., and
  Oh, A. (eds.), \emph{Advances in Neural Information Processing Systems},
  volume~35, pp.\  1362--1375. Curran Associates, Inc., 2022.

\bibitem[Maggs et~al.(2024)Maggs, Hacker, and Rieck]{Maggs24a}
Maggs, K., Hacker, C., and Rieck, B.
\newblock Simplicial representation learning with neural $k$-forms.
\newblock In \emph{International Conference on Learning Representations}, 2024.
\newblock URL \url{https://openreview.net/forum?id=Djw0XhjHZb}.

\bibitem[Marsh \& Beers(2023)Marsh and Beers]{Marsh23a}
Marsh, L. and Beers, D.
\newblock Stability and inference of the {E}uler characteristic transform.
\newblock \emph{arXiv:2303.13200}, 2023.

\bibitem[Marsh et~al.(2024)Marsh, Zhou, Qin, Lu, Byrne, and
  Harrington]{marsh2024detecting}
Marsh, L., Zhou, F.~Y., Qin, X., Lu, X., Byrne, H.~M., and Harrington, H.~A.
\newblock Detecting temporal shape changes with the {E}uler characteristic
  transform.
\newblock \emph{Transactions of Mathematics and Its Applications}, 8\penalty0
  (2), 2024.

\bibitem[Maurya et~al.(2021)Maurya, Liu, and Murata]{maurya2021improving}
Maurya, S.~K., Liu, X., and Murata, T.
\newblock Improving graph neural networks with simple architecture design.
\newblock \emph{arXiv:2105.07634}, 2021.

\bibitem[Mernyei \& Cangea(2020)Mernyei and Cangea]{mernyei2020wiki}
Mernyei, P. and Cangea, C.
\newblock Wiki-cs: A wikipedia-based benchmark for graph neural networks.
\newblock \emph{arXiv preprint arXiv:2007.02901}, 2020.

\bibitem[Morris et~al.(2023)Morris, Lipman, Maron, Rieck, Kriege, Grohe, Fey,
  and Borgwardt]{Morris23a}
Morris, C., Lipman, Y., Maron, H., Rieck, B., Kriege, N.~M., Grohe, M., Fey,
  M., and Borgwardt, K.
\newblock {W}eisfeiler and {L}eman go machine learning: The story so far.
\newblock \emph{Journal of Machine Learning Research}, 24\penalty0
  (333):\penalty0 1--59, 2023.

\bibitem[M{\"u}ller et~al.(2024)M{\"u}ller, Galkin, Morris, and
  Ramp{\'a}{\v{s}}ek]{mueller2024attending}
M{\"u}ller, L., Galkin, M., Morris, C., and Ramp{\'a}{\v{s}}ek, L.
\newblock Attending to graph transformers.
\newblock \emph{Transactions on Machine Learning Research}, 2024.
\newblock URL \url{https://openreview.net/forum?id=HhbqHBBrfZ}.

\bibitem[Papamarkou et~al.(2024)Papamarkou, Birdal, Bronstein, Carlsson, Curry,
  Gao, Hajij, Kwitt, Liò, Lorenzo, Maroulas, Miolane, Nasrin, Ramamurthy,
  Rieck, Scardapane, Schaub, Veličković, Wang, Wang, Wei, and
  Zamzmi]{Papamarkou24a}
Papamarkou, T., Birdal, T., Bronstein, M., Carlsson, G., Curry, J., Gao, Y.,
  Hajij, M., Kwitt, R., Liò, P., Lorenzo, P.~D., Maroulas, V., Miolane, N.,
  Nasrin, F., Ramamurthy, K.~N., Rieck, B., Scardapane, S., Schaub, M.~T.,
  Veličković, P., Wang, B., Wang, Y., Wei, G.-W., and Zamzmi, G.
\newblock Position: Topological deep learning is the new frontier for
  relational learning.
\newblock In Salakhutdinov, R., Kolter, Z., Heller, K., Weller, A., Oliver, N.,
  Scarlett, J., and Berkenkamp, F. (eds.), \emph{Proceedings of the 41st
  International Conference on Machine Learning}, volume 235 of
  \emph{Proceedings of Machine Learning Research}, pp.\  39529--39555. PMLR,
  2024.

\bibitem[Pei et~al.(2020)Pei, Wei, Chang, Lei, and Yang]{pei2020geom}
Pei, H., Wei, B., Chang, K. C.-C., Lei, Y., and Yang, B.
\newblock {Geom-GCN}: {G}eometric graph convolutional networks.
\newblock In \emph{International Conference on Learning Representations}, 2020.
\newblock URL \url{https://openreview.net/forum?id=S1e2agrFvS}.

\bibitem[Platonov et~al.(2023)Platonov, Kuznedelev, Diskin, Babenko, and
  Prokhorenkova]{platonov2023critical}
Platonov, O., Kuznedelev, D., Diskin, M., Babenko, A., and Prokhorenkova, L.
\newblock A critical look at the evaluation of {GNNs} under heterophily: Are we
  really making progress?, 2023.
\newblock URL \url{https://openreview.net/forum?id=tJbbQfw-5wv}.

\bibitem[Ramp\'{a}\v{s}ek et~al.(2022)Ramp\'{a}\v{s}ek, Galkin, Dwivedi, Luu,
  Wolf, and Beaini]{rampavsek2022recipe}
Ramp\'{a}\v{s}ek, L., Galkin, M., Dwivedi, V.~P., Luu, A.~T., Wolf, G., and
  Beaini, D.
\newblock Recipe for a general, powerful, scalable graph transformer.
\newblock In Koyejo, S., Mohamed, S., Agarwal, A., Belgrave, D., Cho, K., and
  Oh, A. (eds.), \emph{Advances in Neural Information Processing Systems},
  volume~35, pp.\  14501--14515. Curran Associates, Inc., 2022.

\bibitem[Rieck et~al.(2019)Rieck, Bock, and Borgwardt]{Rieck19b}
Rieck, B., Bock, C., and Borgwardt, K.
\newblock A persistent {W}eisfeiler--{L}ehman procedure for graph
  classification.
\newblock In Chaudhuri, K. and Salakhutdinov, R. (eds.), \emph{Proceedings of
  the 36th International Conference on Machine Learning}, number~97 in
  Proceedings of Machine Learning Research, pp.\  5448--5458. PMLR, 2019.

\bibitem[R{\"o}ell \& Rieck(2024)R{\"o}ell and Rieck]{roell2023differentiable}
R{\"o}ell, E. and Rieck, B.
\newblock Differentiable {E}uler characteristic transforms for shape
  classification.
\newblock In \emph{International Conference on Learning Representations}, 2024.
\newblock URL \url{https://openreview.net/forum?id=MO632iPq3I}.

\bibitem[R\"oell \& Rieck(2025)R\"oell and Rieck]{Roell25a}
R\"oell, E. and Rieck, B.
\newblock Point cloud synthesis using inner product transforms.
\newblock \emph{arXiv:2410.18987}, 2025.

\bibitem[Rozemberczki et~al.(2021)Rozemberczki, Allen, and
  Sarkar]{rozemberczki2021multi}
Rozemberczki, B., Allen, C., and Sarkar, R.
\newblock Multi-scale attributed node embedding.
\newblock \emph{Journal of Complex Networks}, 9\penalty0 (2):\penalty0 cnab014,
  2021.

\bibitem[Rusch et~al.(2023)Rusch, Bronstein, and Mishra]{rusch2023survey}
Rusch, T.~K., Bronstein, M.~M., and Mishra, S.
\newblock A survey on oversmoothing in graph neural networks.
\newblock \emph{arXiv:2303.10993}, 2023.

\bibitem[Shchur et~al.(2018)Shchur, Mumme, Bojchevski, and
  G{\"u}nnemann]{shchur2018pitfalls}
Shchur, O., Mumme, M., Bojchevski, A., and G{\"u}nnemann, S.
\newblock Pitfalls of graph neural network evaluation.
\newblock In \emph{Relational Representation Learning Workshop~(R2L) at
  NeurIPS}, 2018.

\bibitem[Shi et~al.(2021)Shi, Huang, Feng, Zhong, Wang, and Sun]{shi2020masked}
Shi, Y., Huang, Z., Feng, S., Zhong, H., Wang, W., and Sun, Y.
\newblock Masked label prediction: Unified message passing model for
  semi-supervised classification.
\newblock In Zhou, Z.-H. (ed.), \emph{Proceedings of the 30th International
  Joint Conference on Artificial Intelligence}, pp.\  1548--1554. International
  Joint Conferences on Artificial Intelligence Organization, 2021.

\bibitem[Southern et~al.(2023)Southern, Wayland, Bronstein, and
  Rieck]{Southern23a}
Southern, J., Wayland, J., Bronstein, M., and Rieck, B.
\newblock Curvature filtrations for graph generative model evaluation.
\newblock In Oh, A., Neumann, T., Globerson, A., Saenko, K., Hardt, M., and
  Levine, S. (eds.), \emph{Advances in Neural Information Processing Systems},
  volume~36, pp.\  63036--63061. Curran Associates, Inc., 2023.

\bibitem[Suresh et~al.(2021)Suresh, Budde, Neville, Li, and
  Ma]{suresh2021breaking}
Suresh, S., Budde, V., Neville, J., Li, P., and Ma, J.
\newblock Breaking the limit of graph neural networks by improving the
  assortativity of graphs with local mixing patterns.
\newblock In \emph{Proceedings of the 27th ACM SIGKDD Conference on Knowledge
  Discovery \& Data Mining}, pp.\  1541--1551, 2021.

\bibitem[Topping et~al.(2022)Topping, Giovanni, Chamberlain, Dong, and
  Bronstein]{topping2022oversquashing}
Topping, J., Giovanni, F.~D., Chamberlain, B.~P., Dong, X., and Bronstein,
  M.~M.
\newblock Understanding over-squashing and bottlenecks on graphs via curvature.
\newblock In \emph{International Conference on Learning Representations}, 2022.
\newblock URL \url{https://openreview.net/forum?id=7UmjRGzp-A}.

\bibitem[Turner et~al.(2014)Turner, Mukherjee, and Boyer]{Turner14a}
Turner, K., Mukherjee, S., and Boyer, D.~M.
\newblock Persistent homology transform for modeling shapes and surfaces.
\newblock \emph{Information and Inference: A Journal of the IMA}, 3\penalty0
  (4):\penalty0 310--344, 12 2014.

\bibitem[Veli{\v{c}}kovi{\'c} et~al.(2018)Veli{\v{c}}kovi{\'c}, Cucurull,
  Casanova, Romero, Lio, and Bengio]{velivckovic2017graph}
Veli{\v{c}}kovi{\'c}, P., Cucurull, G., Casanova, A., Romero, A., Lio, P., and
  Bengio, Y.
\newblock Graph attention networks.
\newblock In \emph{International Conference on Learning Representations}, 2018.
\newblock URL \url{https://openreview.net/forum?id=rJXMpikCZ}.

\bibitem[Veli\v{c}kovi\'{c}(2023)]{velickovic23a}
Veli\v{c}kovi\'{c}, P.
\newblock Everything is connected: Graph neural networks.
\newblock \emph{Current Opinion in Structural Biology}, 79:\penalty0 102538,
  2023.

\bibitem[Verma et~al.(2024)Verma, Souza, and Garg]{Verma24a}
Verma, Y., Souza, A.~H., and Garg, V.
\newblock Topological neural networks go persistent, equivariant, and
  continuous.
\newblock In Salakhutdinov, R., Kolter, Z., Heller, K., Weller, A., Oliver, N.,
  Scarlett, J., and Berkenkamp, F. (eds.), \emph{Proceedings of the 41st
  International Conference on Machine Learning}, volume 235 of
  \emph{Proceedings of Machine Learning Research}, pp.\  49388--49407. PMLR,
  2024.

\bibitem[von Rohrscheidt \& Rieck(2023)von Rohrscheidt and
  Rieck]{vonRohrscheidt23a}
von Rohrscheidt, J. and Rieck, B.
\newblock Topological singularity detection at multiple scales.
\newblock In Krause, A., Brunskill, E., Cho, K., Engelhardt, B., Sabato, S.,
  and Scarlett, J. (eds.), \emph{Proceedings of the 40th International
  Conference on Machine Learning}, volume 202 of \emph{Proceedings of Machine
  Learning Research}, pp.\  35175--35197. PMLR, 2023.

\bibitem[Wang \& Zhang(2022)Wang and Zhang]{wang2022powerful}
Wang, X. and Zhang, M.
\newblock How powerful are spectral graph neural networks.
\newblock In Chaudhuri, K., Jegelka, S., Song, L., Szepesvari, C., Niu, G., and
  Sabato, S. (eds.), \emph{Proceedings of the 39th International Conference on
  Machine Learning}, volume 162 of \emph{Proceedings of Machine Learning
  Research}, pp.\  23341--23362. PMLR, 2022.

\bibitem[Wayland et~al.(2024)Wayland, Coupette, and Rieck]{Wayland24a}
Wayland, J., Coupette, C., and Rieck, B.
\newblock Mapping the multiverse of latent representations.
\newblock In Salakhutdinov, R., Kolter, Z., Heller, K., Weller, A., Oliver, N.,
  Scarlett, J., and Berkenkamp, F. (eds.), \emph{Proceedings of the 41st
  International Conference on Machine Learning}, number 235 in Proceedings of
  Machine Learning Research, pp.\  52372--52402, 2024.

\bibitem[Wu et~al.(2019)Wu, Souza, Zhang, Fifty, Yu, and
  Weinberger]{wu2019simplifying}
Wu, F., Souza, A., Zhang, T., Fifty, C., Yu, T., and Weinberger, K.
\newblock Simplifying graph convolutional networks.
\newblock In Chaudhuri, K. and Salakhutdinov, R. (eds.), \emph{Proceedings of
  the 36th International Conference on Machine Learning}, volume~97 of
  \emph{Proceedings of Machine Learning Research}, pp.\  6861--6871. PMLR,
  2019.

\bibitem[Xu et~al.(2019)Xu, Hu, Leskovec, and Jegelka]{xu2018powerful}
Xu, K., Hu, W., Leskovec, J., and Jegelka, S.
\newblock How powerful are graph neural networks?
\newblock In \emph{International Conference on Learning Representations}, 2019.
\newblock URL \url{https://openreview.net/forum?id=ryGs6iA5Km}.

\bibitem[Xu et~al.(2024)Xu, Huang, and State]{Xu24a}
Xu, Y., Huang, H., and State, R.
\newblock {CTQW-GraphSAGE}: Trainabel continuous-time quantum walk on graph.
\newblock In \emph{Proceedings of the 33rd International Conference on
  Artificial Neural Networks}, pp.\  79--92. Springer, 2024.

\bibitem[Yang et~al.(2016)Yang, Cohen, and Salakhudinov]{yang2016revisiting}
Yang, Z., Cohen, W., and Salakhudinov, R.
\newblock Revisiting semi-supervised learning with graph embeddings.
\newblock In Balcan, M.~F. and Weinberger, K.~Q. (eds.), \emph{Proceedings of
  The 33rd International Conference on Machine Learning}, volume~48 of
  \emph{Proceedings of Machine Learning Research}, pp.\  40--48. PMLR, 2016.

\bibitem[Zhang et~al.(2024)Zhang, Xu, He, Guo, and Cui]{zhang2023comprehensive}
Zhang, X., Xu, Y., He, W., Guo, W., and Cui, L.
\newblock A comprehensive review of the oversmoothing in graph neural
  networks.
\newblock In Sun, Y., Lu, T., Wang, T., Fan, H., Liu, D., and Du, B. (eds.),
  \emph{Computer Supported Cooperative Work and Social Computing}, pp.\
  451--465. Springer, 2024.

\bibitem[Zhao \& Wang(2019)Zhao and Wang]{zhao2019learning}
Zhao, Q. and Wang, Y.
\newblock Learning metrics for persistence-based summaries and applications for
  graph classification.
\newblock In Wallach, H., Larochelle, H., Beygelzimer, A., d\textquotesingle
  Alch\'{e}-Buc, F., Fox, E., and Garnett, R. (eds.), \emph{Advances in Neural
  Information Processing Systems}, volume~32. Curran Associates, Inc., 2019.

\bibitem[Zhao et~al.(2020)Zhao, Ye, Chen, and Wang]{zhao2020persistence}
Zhao, Q., Ye, Z., Chen, C., and Wang, Y.
\newblock Persistence enhanced graph neural network.
\newblock In \emph{International Conference on Artificial Intelligence and
  Statistics}, pp.\  2896--2906. PMLR, 2020.

\bibitem[Zheleva \& Getoor(2009)Zheleva and Getoor]{zheleva2009join}
Zheleva, E. and Getoor, L.
\newblock To join or not to join: the illusion of privacy in social networks
  with mixed public and private user profiles.
\newblock In \emph{Proceedings of the 18th International Conference on World
  Wide Web}, pp.\  531--540, 2009.

\bibitem[Zhu et~al.(2020)Zhu, Yan, Zhao, Heimann, Akoglu, and
  Koutra]{zhu2020beyond}
Zhu, J., Yan, Y., Zhao, L., Heimann, M., Akoglu, L., and Koutra, D.
\newblock Beyond homophily in graph neural networks: Current limitations and
  effective designs.
\newblock In Larochelle, H., Ranzato, M., Hadsell, R., Balcan, M., and Lin, H.
  (eds.), \emph{Advances in Neural Information Processing Systems}, volume~33,
  pp.\  7793--7804. Curran Associates, Inc., 2020.

\bibitem[Zhu et~al.(2021)Zhu, Rossi, Rao, Mai, Lipka, Ahmed, and
  Koutra]{zhu2021graph}
Zhu, J., Rossi, R.~A., Rao, A., Mai, T., Lipka, N., Ahmed, N.~K., and Koutra,
  D.
\newblock Graph neural networks with heterophily.
\newblock In \emph{Proceedings of the AAAI Conference on Artificial
  Intelligence}, volume~35, pp.\  11168--11176, 2021.

\end{thebibliography}
\bibliographystyle{icml2025}

\appendix
\onecolumn
\section{Appendix}\label{sec:appendix}

\subsection{Proofs}

We briefly restate all theorem from the main text for the reader's
convenience before providing proofs.

\ThmExpressivity*

\begin{proof}
By the remark in \cref{sec:methods}~(in the paragraph right above the
original statement of this Theorem), we may assume that the natural
embedding of $\mathcal{G}$ into $\mathbb{R}^n$ is a graph isomorphism.
Then, making use of the invertibility theorem, the $1$-hop neighborhood
of a point $x$ in the embedding of $\mathcal{G}$ can be reconstructed
from $\lECT_1(x;\mathcal{G})$. Therefore, the feature vectors of $x$ and
its $1$-hop neighbors can be deduced from $\lECT_1(x;\mathcal{G})$,
which is the only non-learnable information one needs to perform
a message-passing step.
\end{proof}

\ThmIsomorphism*

\begin{proof}
    When two featured graphs are isomorphic in the sense of
    \cref{def:isomorphic_graphs}, their respective Euclidean
    embeddings produce equal $\ECT$s by construction because the node
    feature vectors of two corresponding points under the isomorphism
    are equal. By contrast, let us assume that
    $\ECT(\mathcal{G}_1)=\ECT(\mathcal{G}_2)$. Then by the invertibility
    theorem, the Euclidean embeddings of $\mathcal{G}_1$ and
    $\mathcal{G}_2$ are equal. Therefore, the only information that may
    tell apart the two graphs are their node labels, but this means that
    $\mathcal{G}_1$ and $\mathcal{G}_2$ are isomorphic.
\end{proof}

\ThmMetric*

\begin{proof}
    $d_{\ECT}(X,X)=0$ holds for $\rho$ being the identity. Now assume
    that $d_{\ECT}(X,Y)=0$. Then there exists $\rho \in \mathrm{SO}(n)$ with
    $\norm{(\ECT(X)- \ECT(\rho Y))}_{\infty}=0$. As
    $\norm{\bullet}_{\infty}$ is a norm, it follows that
    $\ECT(X)=\ECT(\rho Y)$, and by the invertibility theorem we obtain
    $X = \rho Y$. This shows the first property of a metric~(note that
    positivity follows from $\norm{\bullet}_{\infty}$). For symmetry,
    note that $\norm{(\ECT(X)- \ECT(\rho
    Y))}_{\infty}=\norm{(\ECT(\rho^{-1} X)- \ECT(Y))}_{\infty}$ since
    rotations are invertible. For the triangle inequality, let $Z$ be
    another finite simplicial complex. $d_{\ECT}(X,Z) $ then reads
    $\inf_{\rho \in \mathrm{SO}(n)} \norm{(\ECT(X)- \ECT(\rho
      Z))}_{\infty}$, which is less than or equal to $\inf_{\rho,\rho'
        \in \mathrm{SO}(n)} (\norm{(\ECT(X)- \ECT(\rho' Y))}_{\infty}+
        \norm{(\ECT(\rho'Y)- \ECT(\rho Z))}_{\infty})$. This term,
        however, is equal to $\inf_{\rho,\rho' \in \mathrm{SO}(n)} (\norm{(\ECT(X)- \ECT(\rho' Y))}_{\infty}+ \norm{(\ECT(Y)- \ECT((\rho')^{-1}
        \rho Z))}_{\infty})$, which is equal to $\inf_{\rho \in SO(n)}
        \norm{(\ECT(X)- \ECT(\rho Y))}_{\infty} + \inf_{\rho \in
        \mathrm{SO}(n)} \norm{(\ECT(Y)- \ECT(\rho Z))}_{\infty}$. But
        this final term is \emph{precisely} the definition of $d_{\ECT}(X,Y) + d_{\ECT}(Y,Z)$.
\end{proof}

\subsection{Computational Complexity}\label{app:complexity}

For a fixed node $x$, the computational complexity of $\lECT_k(x)$ is $O(m \cdot l \cdot |N_k(x)|)$,
where:
\begin{inparaenum}[(i)]
    \item \(m\) is the number of sampled directions,
    \item \(l\) is the number of filtration steps, and
    \item \(|N_k(x)|\) is the number of vertices (or simplices) in the \(k\)-hop neighborhood of \(x\).
\end{inparaenum}

\clearpage

\subsection{Ablation on Directions and Interpretability}
\label{app:ablation}

\begin{figure}[tbp]
    \centering
    \includegraphics[width=0.5\textwidth]{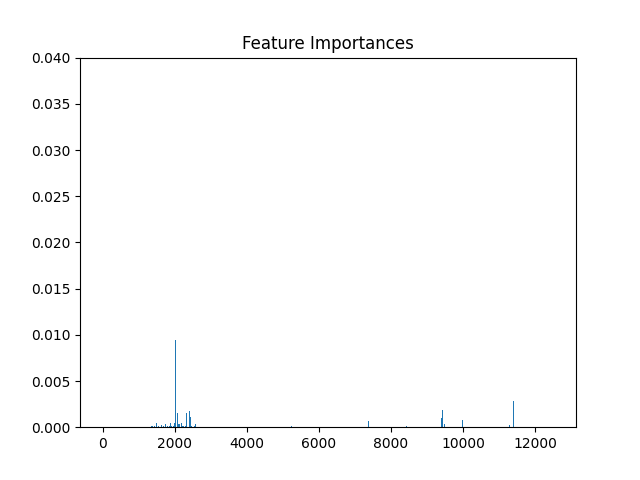}
    \caption{%
      Feature importance scores of an XGBoost model for the ``Coauthor
      Physics'' dataset (using $\lECT_1$). Only a small number of
      features admit high importance scores.
    }
    \label{fig:feature_importances}
\end{figure}

Coming back to our approximation of $\ECT(X)$ via
$\overline{\ECT}(X)_{(m,l)} :=\ECT(X)_{\vert \{ v_1,\dots,v_m \} \times
\{ t_1,\dots,t_l \} }$ for uniformly-distributed directions
$v_1,\dots,v_m \in S^{n-1}$ and filtration steps $t_1, \dots , t_l \in
\mathbb{R}$, we notice that the $(l \cdot (j-1)+1)$-th till $(l \cdot
j)$-th entries of $\overline{\ECT}(X)_{(m,l)}$ correspond to the
direction $v_j$.
The latter gives us the opportunity to get a deeper understanding of how
the model predicts its outcome, by analyzing its feature importance
values~(which are available for tree-based algorithms like XGBoost). 
Therefore, our approach enables us to analyze which features, i.e.,
directions, of the underlying ECT vector are most important. 
In practice, we often observe that a small number of features admits
high feature importance with respect to the corresponding model~(cf.\
\Cref{fig:feature_importances}). This raises the question if we may
use a \emph{smaller} random collection of features and still obtain
reasonably useful results.
We therefore ran experiments for a collection of
datasets for a varying number of randomly-sampled entries of the
$\lECT_1$ vector; cf.\ \cref{tab:ablation_results}.
Here, 4096 corresponds to the whole vector. We observe that for certain
datasets, such as ``Coauthor CS,'' ``Coauthor Physics,'' and ``Amazon
Ratings,'' the
performance of the model only slightly changes when using a reduced
version of the  $\lECT_1$ vector. In light of the results by
\citet{curry2022many}, this observation is not entirely surprising---one
main claim therein is that the $\ECT$ can be determined using
a small number of directions.

\begin{table}[b]
\centering
\small
\setlength{\tabcolsep}{3pt}
\caption{%
  Mean accuracy~(in percent, $5$ runs each) for different
  node-classification tasks, and varying numbers of randomly-sampled
  entries of the corresponding $\lECT_1$ vectors.
}
\sisetup{
  table-alignment = left,
}
\begin{tabular}{lSSSSSS}
\toprule
\textbf{Dataset} & \textbf{0} & \textbf{50} & \textbf{100} & \textbf{500} & \textbf{1000} & \textbf{4096} \\
\midrule
WikiCS & 67.8 & 69.2 & 70.5 & 71.3 & 72.7 & 74.6 \\
Coauthor CS & 92.1 & 92.3 & 92.4 & 92.5 & 92.6 & 92.6 \\
Coauthor Physics & 95.2 & 95.6 & 95.6 & 95.8 & 95.9 & 96.1 \\
Roman Empire & 64.7 & 73.7 & 75.8 & 78.3 & 79.7 & 80.4 \\
Amazon Ratings & 47.9 & 47.9 & 48.2 & 48.4 & 48.2 & 48.4 \\
\bottomrule
\end{tabular}
\label{tab:ablation_results}
\end{table}

\clearpage

\subsection{Additional Node Classification Experiment}\label{app:additional_node_classification}

\paragraph{WikiCS Dataset}
To further validate the effectiveness of our approach, we consider the WikiCS dataset~\citep{mernyei2020wiki}, a medium-sized co-occurrence graph derived from Wikipedia articles on computer science topics. Nodes represent articles, and edges reflect mutual links between them. Each node is equipped with a 300-dimensional embedding, and the task is to classify articles into one of several predefined categories.
As shown in \cref{tab:ml_performance_wikics}, our $\lECT_1$-based method achieves competitive performance compared to message-passing baselines. While GAT obtains the best accuracy overall, $\lECT_1$ performs on par with H2GCN and GCN, despite not relying on neighborhood aggregation. This supports the idea that $\lECT$-based representations can serve as effective input features in classification settings, even for graphs with moderately homophilous structures. The small standard deviation further illustrates the stability of our method across splits.

\begin{table}[tbp]
\centering
\setlength{\tabcolsep}{3pt} 
\let\b\bfseries
\small
\caption{%
  Performance (\emph{accuracy}, in percent) of graph-learning
  models on WikiCS dataset~($5$ training splits).
}
\label{tab:ml_performance_wikics}
\begin{tabular}{@{}l l@{}}
\toprule
\textbf{Model} & \textbf{WikiCS} \\
\midrule
GCN               &  75.2 $\pm$ 0.8 \\
GAT               &  78.7 $\pm$ 1.2 \\
GIN               &  74.2 $\pm$ 1.7 \\
GraphSAGE         &  73.4 $\pm$ 1.5 \\
H2GCN             &  75.3 $\pm$ 1.4 \\
\midrule
$\lECT_1$         &  74.6 $\pm$ 0.5 \\
\bottomrule
\end{tabular}
\end{table}

\subsection{Spatial Alignment of High-Dimensional Data}
\label{app:Spatial Alignment}

\begin{figure}[tbp]
  \centering
  \begin{tikzpicture}
    \pgfplotsset{%
      boxplot/draw direction = x,
      boxplot/whisker range = 10000, 
    }
    \begin{axis}[%
      axis x line*     = bottom,
      axis y line*     = none,
      ytick            = {1, 2},
      yticklabels      = {non-aligned, aligned},
      height           = 3.0cm,
      width            = \linewidth,
      mark size        = 1pt,
      tick align       = outside,
      xlabel           = {Squared $L^2$ distance},
      ylabel           = {},
      xticklabel style = {%
        /pgf/number format/fixed,
        /pgf/number format/precision=2,
      },
    ]
      \addplot[boxplot] table[y index = 0] {Data/distances_non_aligned_mnist_ones.txt};
      \addplot[boxplot] table[y index = 0] {Data/distances_aligned_mnist_ones.txt};
    \end{axis}
  \end{tikzpicture}%
  \caption{%
    A comparison of the squared $L^2$ distances of $\lECT$s of aligned
    and non-aligned MNIST digits of ``1,'' respectively.
  }
  \label{fig:losses_MNIST_aligned_vs_non_aligned}
\end{figure}

Following our previous observations that $d_{\ECT}$ enables us to align
two spaces, we now use it to investigate its effect on high-dimensional
data. We start this discussion with the well-known MNIST benchmark
dataset, following an analysis of local geometrical-topological
structures that we performed previously~\citep{vonRohrscheidt23a}.
We thus first represent each~(gray-scale) image in the dataset as
a $784$-dimensional vector, by flattening the image. In this way, we
obtain a high-dimensional point cloud corresponding to the dataset.
Subsequently, we sample $300$ points of digits of ``$1$'' and calculate
the pairwise distances of their respective $\lECT$~(with respect to the
whole point cloud), for $k=10$. Finally, we calculate the pairwise
distances of the respective aligned $\lECT$s~(by using the approach of
\cref{eq:learning_metric} with $k=10$).
\Cref{fig:losses_MNIST_aligned_vs_non_aligned} shows the results; we
observe that the aligned $\lECT$s have a significantly lower squared $L^2$
distance~(with a median of $\approx 112$) than the non-aligned
ones~(with a median of $\approx 224$), showcasing that rotations cause
dissimilarity between small neighborhoods of points, in many cases.

\subsection{Homophily Scores of Node-Classification Experiments}

\Cref{tab:edge_homophily} reveals that our benchmark suite spans the
entire range from extreme homophily to extreme heterophily.  
At the homophilic end lie the Amazon co-purchase graphs ``Computers''
and ``Photo,'' together with the ``Planetoid'' citation graphs ``Cora,''
``CiteSeer,'' and ``PubMed.'' In each of these networks, at least
\emph{seven of every ten} edges connect nodes with identical
class labels, replicating the conditions under which early
message-passing GNNs achieved their seminal successes.  
Near the middle of the spectrum, ``Cornell'' and the ``Amazon Ratings''
datasets exhibit mixed behavior, with roughly one third of their edges
being heterophilous. Thus, neighborhood aggregation still conveys useful
class-specific information, but the signal is noticeably diluted.  
The lower end is populated by the remaining datasets, where fewer than
\emph{one edge in three} is homophilic, and by ``Roman Empire,'' the most
extreme case in our experimental suite, where only about \emph{one edge
in twenty} links same-label endpoints.  

\begin{wraptable}{l}{0.45\linewidth}
  \centering
  \small
  \setlength{\tabcolsep}{3pt}
  \caption{%
  Edge-homophily ratios for every dataset used in our
  node-classification experiments, sorted in ascending order. As the
  table shows, our experiments comprise a wide variety of datasets.
}
  \label{tab:edge_homophily}
  \begin{tabular}{lc}
    \toprule
    \textbf{Dataset} & \textbf{$H_{\text{edge}}$} \\
    \midrule
    Roman Empire            & 0.047 \\
    Texas                   & 0.110 \\
    Wisconsin               & 0.210 \\
    Actor                   & 0.217 \\
    Squirrel                & 0.220 \\
    Chameleon               & 0.230 \\
    Cornell                 & 0.300 \\
    Amazon Ratings          & 0.380 \\
    CiteSeer                & 0.736 \\
    Amazon Computers        & 0.777 \\
    PubMed                  & 0.802 \\
    Cora                    & 0.810 \\
    Amazon Photo            & 0.827 \\
    \bottomrule
  \end{tabular}
\end{wraptable}
Because six datasets are heterophilic, five are strongly homophilic, and
two occupy the transition zone, we believe our experimental suite to be
effectively \emph{balanced}. A model must therefore operate reliably
across sharply different structural regimes to achieve consistently high
average rank.  Classical message-passing architectures depend on
\emph{homophily} and tend to deteriorate as the ratio falls, whereas our
empirical results from \cref{sec:Experiments} demonstrate that our
proposed $\lECT$ representations retain competitive---and often even
\emph{superior}---accuracy \emph{regardless of homophily level}.
The most conspicuous gains appear precisely on the graphs where
neighbor aggregation is \emph{least informative}, namely ``Roman
Empire,'' ``Texas,'' and ``Wisconsin,'' confirming that $\lECT$ features
capture structural cues that message passing alone fails to exploit.
Hence, the numerical landscape mapped out in \Cref{tab:edge_homophily}
substantiates the claim that our experimental design both stresses the
limits of common GNNs while at the same time showcasing the robustness
of $\lECT$-based approaches in settings where label agreement along
edges is sparse.

\clearpage

\subsection{Post-hoc Evaluation of Node-Classification Experiments}\label{app:evaluation}

A critical difference diagram arranges the average ranks of multiple
models across a set of datasets in order to facilitate overall
performance comparisons between the model performances.
Such diagrams are commonly used when comparing a suite of models on
different datasets~(cf.\ \citet{Borgwardt20} for similar plots in the
context of \emph{graph kernels}).

\Cref{fig:cd-diagram} shows the results for all node-classification
results from \Cref{sec:methods}, including both homophilic and
heterophilic graph datasets. We observe that the $\lECT$-based approaches
outperform standard methods and the heterophily-specific architecture
H2GCN by far, when averaged over all datasets.\footnote{%
  We used
\url{https://github.com/hfawaz/cd-diagram} for the creation of the
critical difference diagram.}
The best-performing method $\lECT_1 + \lECT_2$ exhibits an average rank
of $2$, while the worst performing method is GIN with an average rank of
$5.7$. Even the worst-performing $\lECT$-based method ($\lECT_2$)
performs better than the best non-$\lECT$-based method, i.e., GAT.
However, the most interesting fact that can be gleaned from the diagram
involves the statistical significance of the results. Methods connected
over the same bar are not performing statistically significantly
differently. This seemingly \emph{negative} result has a positive
implication: Despite being orders of magnitude more complex, even
specialized graph neural networks do not perform statistically
significantly better than $\lECT$-based methods. Given that our results
are based on a standard XGBoost model without \emph{any} task-specific
hyperparameter tuning, we believe that this demonstrates the potential
and practical utility of our proposed methods.

\begin{figure}[tbp]
    \centering
    \includegraphics[width=0.5\textwidth]{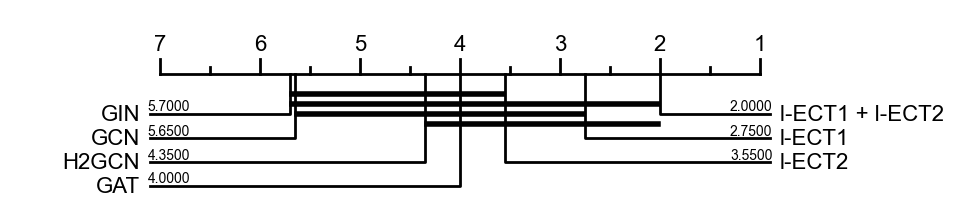}
    \caption{%
      \emph{Critical difference diagram} showing the ranks of different models across
      all node-classification tasks from \Cref{sec:Experiments}. Even
      the worst-performing $\lECT$-based approach~($\lECT_2$) exhibits
      superior performance to \emph{all other methods}, when averaged
      across all tasks.
    }
    \label{fig:cd-diagram}
\end{figure}

\begin{wraptable}{l}{0.45\linewidth}
    \centering
    \sisetup{
      table-alignment = left,
      round-mode      = places,
      table-format    = 2.3,
      round-precision = 1,
    }
    \small
    \setlength{\tabcolsep}{3pt}
    \caption{%
    Ranks~(lower is better) of models from \citet{platonov2023critical}
    across the \emph{heterophilic} datasets therein, in comparison to our
    methods. Notice that our method is a \emph{general-purpose} method
    for node classification and neither geared towards heterophily nor
    homophily.
  }
  \label{tab:ranks_critical}
    \begin{tabular}{lS}
        \toprule
        \textbf{Model} & \textbf{Rank} \\
        \midrule
        H2GCN \citep{zhu2020beyond}              & 18.250 \\
        CPGNN \citep{zhu2021graph}              & 16.750 \\
        GPR-GNN \citep{chien2020adaptive}        & 15.250 \\
        ResNet \citep{he2016deep}               & 13.750 \\
        l-ECT1                                   & 12.375 \\
        l-ECT2                                   & 12.375 \\
        GAT \citep{velivckovic2017graph}         & 12.250 \\
        GT \citep{shi2020masked}                 & 11.000 \\
        \textbf{l-ECT1 + l-ECT (ours)}           & \bfseries 11.000 \\
        GloGNN \citep{li2022finding}             & 11.000 \\
        ResNet+SGC \citep{wu2019simplifying}     & 10.750 \\
        FAGCN \citep{bo2021beyond}              & 10.000 \\
        JacobiConv \citep{wang2022powerful}      & 9.750 \\
        GCN \citep{kipf2016semi}                & 9.625 \\
        GBK-GNN \citep{du2022gbk}                & 9.000 \\
        ResNet+adj \citep{zheleva2009join}       & 7.250 \\
        SAGE \citep{hamilton2017inductive}       & 5.875 \\
        GAT-sep \citep{velivckovic2017graph}     & 5.500 \\
        GT-sep \citep{shi2020masked}            & 5.250 \\
        FSGNN \citep{maurya2021improving}        & 3.000 \\
        \bottomrule
    \end{tabular}
\end{wraptable}
To further evaluate the performance of our methods in comparison to those
reported in the literature, we also included a comparison with the results
presented by \citet{platonov2023critical}, using the ranks of the
respective models as the basis for evaluation; cf.\
\cref{tab:ranks_critical}.
Among the
listed methods, several, such as H2GCN, CPGNN, and GPR-GNN, are
explicitly designed for heterophilic graph settings, leveraging
specialized architectures to handle the challenges posed by such data.
In contrast, our $\lECT_1 + \lECT_2$ method, despite being
a general-purpose approach not tailored specifically for heterophilic
settings, achieves a competitive rank of $11$. This performance is on
a par with other top-performing heterophily-specific models, such as
GloGNN, and outperforms well-established architectures like GT and GAT
by a significant margin. Overall, these results highlight the robustness
and adaptability of our method, demonstrating its ability to handle
diverse graph structures effectively without requiring customization for
heterophilic scenarios. In consideration of the results given in
\Cref{fig:cd-diagram}, this makes $\lECT$-based approaches
a versatile general-purpose solution for node
-classification
tasks.

\clearpage

\subsection{Spatial Alignment of Wedged Spheres}

\begin{figure}[tbp]
  \centering
  \begin{tikzpicture}
    \pgfplotsset{%
      boxplot/draw direction = x,
      boxplot/whisker range = 10000, 
    }
    \begin{axis}[%
      axis x line*     = bottom,
      axis y line*     = none,
      ytick            = {1, 2},
      yticklabels      = {non-aligned, aligned},
      height           = 3.0cm,
      width            = \linewidth,
      mark size        = 1pt,
      tick align       = outside,
      xlabel           = {Squared $L^2$ distance},
      ylabel           = {},
      xticklabel style = {%
        /pgf/number format/fixed,
        /pgf/number format/precision=2,
      },
    ]
      \addplot[boxplot] table[y index = 0] {Data/distances_wedged_spheres_non-aligned.txt};
      \addplot[boxplot] table[y index = 0] {Data/distances_wedged_spheres_aligned.txt};
    \end{axis}
  \end{tikzpicture}%
  \caption{%
    A comparison of the squared $L^2$ distances of the $\ECT$s of
    aligned and non-aligned wedged spheres, respectively.
    We see that alignment results in a median loss of zero, thus
    effectively showing that the two spaces are the same.
  }
  \label{fig:losses_aligned_vs_non_aligned}
\end{figure}

We approximate the optimization problem from \cref{eq:learning_metric}
to show that we can learn a spatial alignment of two data spaces,
while the distance between $\ECT$s of non-aligned spaces that only
differ up to a rotation will generally be high. We start with
a so-called \emph{wedged sphere}, meaning two $2$-dimensional spheres which are
concatenated at a gluing point~(cf.\
\Cref{fig:wedged_spheres_rotated_vs_rerotated} and
\citealt{vonRohrscheidt23a}).
We use $2000$
uniformly-sampled points from such a wedged sphere, and compare the
squared $L^2$ loss between the $\ECT$s of this sample and a rotation of the
same data space. We repeat this procedure $500$ times, where at each
step both the sample of the wedged sphere and the rotation matrix which
yields the rotated version of the same space are sampled randomly.
We notice that the $L^2$ losses between the non-aligned spaces are high
(with a median of around $19$), whereas the $L^2$ losses of the non-aligned
spaces are significantly lower, with a median loss close to zero~(cf.\
\Cref{fig:losses_aligned_vs_non_aligned}).
Moreover, we observe that the $\ECT$ of the same space
significantly changes when the coordinate system is transformed, which
corroborates the necessity of a rotation-invariant metric for the comparison
of $\ECT$s.
We conclude that an alignment of the $\ECT$s of the two underlying data
spaces in fact leads to an alignment of the data spaces itself, as
promised by the theoretical results in \Cref{sec:methods}.

\begin{figure*}[tbp]
    \begin{minipage}[t]{0.5\textwidth}
        \centering
\begin{tikzpicture}
        \begin{axis}[
            width=8cm, 
            height=8cm, 
            view={-50}{50}, 
            grid=both,
            xlabel={$x$},
            ylabel={$y$},
            zlabel={$z$},
            xtick = \empty,
            ytick = \empty,
            ztick = \empty,
            xmin=-2, xmax=2,
            ymin=-2, ymax=2,
            zmin=-2, zmax=2,
            legend style={at={(1,1)}, anchor=north east, font=\small},
            legend cell align={left},
        ]
        
        \addplot3[
            only marks,
            mark=*,
            mark size=1pt,
            color=bleu
        ]
        table[
            col sep=comma,
            x=x,
            y=y,
            z=z
        ] {Data/wedges_spheres_original.csv};
        \addlegendentry{Original point cloud}
        
        \addplot3[
            only marks,
            mark=*,
            mark size=1pt,
            color=cardinal
        ]
        table[
            col sep=comma,
            x=x,
            y=y,
            z=z
        ] {Data/wedges_spheres_rotated.csv};
        \addlegendentry{Rotated point cloud}
        
        \end{axis}
    \end{tikzpicture}
    \end{minipage}%
    \begin{minipage}[t]{0.5\textwidth}
        \centering
\begin{tikzpicture}
        \begin{axis}[
            width=8cm, 
            height=8cm, 
            view={-50}{50}, 
            grid=both,
            xlabel={$x$},
            ylabel={$y$},
            zlabel={$z$},
            xtick = \empty,
            ytick = \empty,
            ztick = \empty,
            xmin=-2, xmax=2,
            ymin=-2, ymax=2,
            zmin=-2, zmax=2,
            legend style={at={(1,1)}, anchor=north east, font=\small},
            legend cell align = left,
        ]
        
        \addplot3[
            only marks,
            mark=*,
            mark size=1pt,
            color=bleu
        ]
        table[
            col sep=comma,
            x=x,
            y=y,
            z=z
        ] {Data/wedges_spheres_original.csv};
        \addlegendentry{Original point cloud}
        
        \addplot3[
            only marks,
            mark=*,
            mark size=1pt,
            color=cardinal
        ]
        table[
            col sep=comma,
            x=x,
            y=y,
            z=z
        ] {Data/wedges_spheres_rerotated.csv};
        \addlegendentry{Re-Rotated point cloud}
        
        \end{axis}
    \end{tikzpicture}
    \end{minipage}
    \caption{A comparison of two wedged spheres with one being rotated around the wedge point and the points being perturbed by Gaussian noise (left) and the learned re-rotated sphere that is aligned with the original data (right).}
        \label{fig:wedged_spheres_rotated_vs_rerotated}
\end{figure*}

\paragraph{Robustness}
\Cref{fig:wedged_spheres_outliers_rotated_vs_rerotated} and
\Cref{fig:wedged_spheres_noise_rotated_vs_rerotated} show that the
spatial alignment of wedged spheres still works satisfactorily, even in
the presence of outliers and noise. 
This property is an important feature when dealing with real-world data,
which is often noisy, and enables us to align spaces that only
\emph{approximately} differ up to a rotation.
By contrast, the Hausdorff distance, i.e., a widely-used metric
between point clouds is~(by definition) \emph{highly sensitive} to outliers.
We therefore conclude that the proposed metric based on $\ECT$s is
a robust metric to compare point clouds of potentially different
cardinalities.

\begin{figure*}[htbp!]
    \begin{minipage}[t]{0.5\textwidth}
        \centering
        \ifarXiv%
          \includegraphics{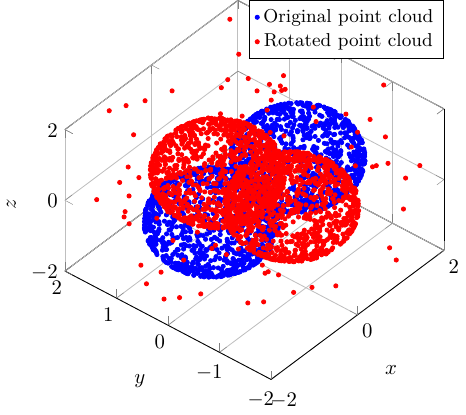}%
        \else%
\begin{tikzpicture}
        \begin{axis}[
            width=8cm, 
            height=8cm, 
            view={-50}{50}, 
            grid=both,
            xlabel={$x$},
            ylabel={$y$},
            zlabel={$z$},
            xtick = \empty,
            ytick = \empty,
            ztick = \empty,
            xmin=-2, xmax=2,
            ymin=-2, ymax=2,
            zmin=-2, zmax=2,
            legend style={at={(1,1)}, anchor=north east, font=\small}
        ]
        
        \addplot3[
            only marks,
            mark=*,
            mark size=1pt,
            color=bleu
        ]
        table[
            col sep=comma,
            x=x,
            y=y,
            z=z
        ] {Data/wedges_spheres_outliers.csv};
        \addlegendentry{Original point cloud}
        
        \addplot3[
            only marks,
            mark=*,
            mark size=1pt,
            color=cardinal
        ]
        table[
            col sep=comma,
            x=x,
            y=y,
            z=z
        ] {Data/wedges_spheres_outliers_rotated.csv};
        \addlegendentry{Rotated point cloud}
        
        \end{axis}
    \end{tikzpicture}
\fi%
    \end{minipage}%
    \begin{minipage}[t]{0.5\textwidth}
        \centering
        \ifarXiv%
          \includegraphics{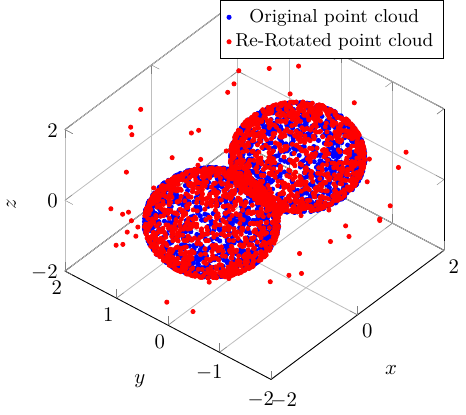}%
        \else%
\begin{tikzpicture}
        \begin{axis}[
            width=8cm, 
            height=8cm, 
            view={-50}{50}, 
            grid=both,
            xlabel={$x$},
            ylabel={$y$},
            zlabel={$z$},
            xtick = \empty,
            ytick = \empty,
            ztick = \empty,
            xmin=-2, xmax=2,
            ymin=-2, ymax=2,
            zmin=-2, zmax=2,
            legend style={at={(1,1)}, anchor=north east, font=\small}
        ]
        
        \addplot3[
            only marks,
            mark=*,
            mark size=1pt,
            color=bleu
        ]
        table[
            col sep=comma,
            x=x,
            y=y,
            z=z
        ] {Data/wedges_spheres_outliers.csv};
        \addlegendentry{Original point cloud}
        
        \addplot3[
            only marks,
            mark=*,
            mark size=1pt,
            color=cardinal
        ]
        table[
            col sep=comma,
            x=x,
            y=y,
            z=z
        ] {Data/wedges_spheres_outliers_rerotated.csv};
        \addlegendentry{Re-Rotated point cloud}
        
        \end{axis}
    \end{tikzpicture}
\fi%
    \end{minipage}
    \caption{A comparison of two wedged spheres, with one being rotated
    around the wedge point and added $200$ outliers (left) and the
  learned re-rotated sphere that is aligned with the original data
(right).}
        \label{fig:wedged_spheres_outliers_rotated_vs_rerotated}
\end{figure*}

\begin{figure*}[htbp!]
    \begin{minipage}[t]{0.5\textwidth}
        \centering
        \ifarXiv%
          \includegraphics{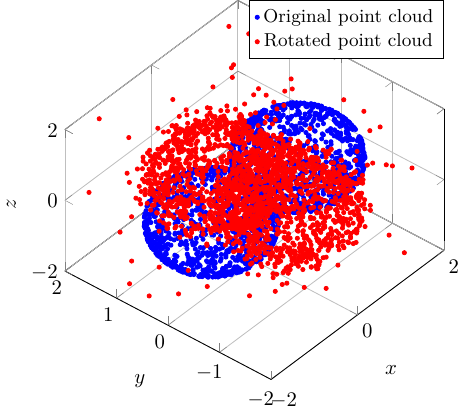}%
        \else%
\begin{tikzpicture}
        \begin{axis}[
            width=8cm, 
            height=8cm, 
            view={-50}{50}, 
            grid=both,
            xlabel={$x$},
            ylabel={$y$},
            zlabel={$z$},
            xtick = \empty,
            ytick = \empty,
            ztick = \empty,
            xmin=-2, xmax=2,
            ymin=-2, ymax=2,
            zmin=-2, zmax=2,
            legend style={at={(1,1)}, anchor=north east, font=\small}
        ]
        
        \addplot3[
            only marks,
            mark=*,
            mark size=1pt,
            color=bleu
        ]
        table[
            col sep=comma,
            x=x,
            y=y,
            z=z
        ] {Data/wedges_spheres_noise.csv};
        \addlegendentry{Original point cloud}
        
        \addplot3[
            only marks,
            mark=*,
            mark size=1pt,
            color=cardinal
        ]
        table[
            col sep=comma,
            x=x,
            y=y,
            z=z
        ] {Data/wedges_spheres_noise_rotated.csv};
        \addlegendentry{Rotated point cloud}
        
        \end{axis}
    \end{tikzpicture}
\fi%
    \end{minipage}%
    \begin{minipage}[t]{0.5\textwidth}
        \centering
        \ifarXiv%
          \includegraphics{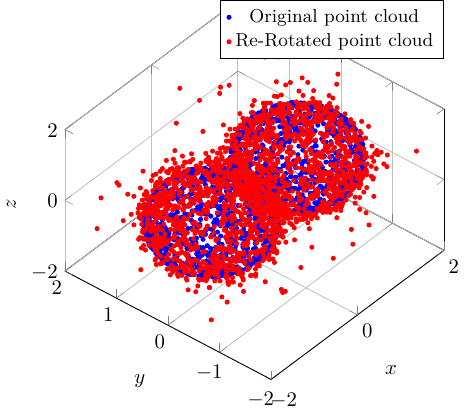}%
        \else%
\begin{tikzpicture}
        \begin{axis}[
            width=8cm, 
            height=8cm, 
            view={-50}{50}, 
            grid=both,
            xlabel={$x$},
            ylabel={$y$},
            zlabel={$z$},
            xtick = \empty,
            ytick = \empty,
            ztick = \empty,
            xmin=-2, xmax=2,
            ymin=-2, ymax=2,
            zmin=-2, zmax=2,
            legend style={at={(1,1)}, anchor=north east, font=\small}
        ]
        
        \addplot3[
            only marks,
            mark=*,
            mark size=1pt,
            color=bleu
        ]
        table[
            col sep=comma,
            x=x,
            y=y,
            z=z
        ] {Data/wedges_spheres_noise.csv};
        \addlegendentry{Original point cloud}
        
        \addplot3[
            only marks,
            mark=*,
            mark size=1pt,
            color=cardinal
        ]
        table[
            col sep=comma,
            x=x,
            y=y,
            z=z
        ] {Data/wedges_spheres_noise_rerotated.csv};
        \addlegendentry{Re-Rotated point cloud}
        
        \end{axis}
    \end{tikzpicture}
\fi%
    \end{minipage}
    \caption{A comparison of two wedged spheres, with one being rotated
    around the wedge point and the points being perturbed by \emph{Gaussian
    noise} (left) and the learned re-rotated sphere that is aligned with
the original data (right).}
        \label{fig:wedged_spheres_noise_rotated_vs_rerotated}
\end{figure*}

\end{document}